\documentclass[sigconf]{acmart}

\AtBeginDocument{%
  \providecommand\BibTeX{{%
    \normalfont B\kern-0.5em{\scshape i\kern-0.25em b}\kern-0.8em\TeX}}}
    
\usepackage{caption}
\usepackage{cleveref}
\usepackage{graphicx}
\usepackage{adjustbox}
\usepackage[most]{tcolorbox}
\usepackage{enumerate}
\usepackage{xcolor}         
\usepackage{subfig}
\usepackage{xcolor,colortbl}
\usepackage{makecell}
\usepackage{tablefootnote}
\usepackage[referable]{threeparttablex}

\newtheorem{theorem}{Theorem}[section]

\setcopyright{acmcopyright}
\copyrightyear{2018}
\acmYear{2018}
\acmDOI{XXXXXXX.XXXXXXX}

\acmConference[Conference acronym 'XX]{Make sure to enter the correct
  conference title from your rights confirmation emai}{June 03--05,
  2018}{Woodstock, NY}

\acmPrice{15.00}
\acmISBN{978-1-4503-XXXX-X/18/06}

\begin{document}

\title{Graph Convolution Networks with Global Random Walk for Directionality and Heterophily}

\author{Wei Zhuo}
\affiliation{%
  \institution{Sun Yat-sen University}
  \city{Shenzhen}
  \country{China}
}
\email{zhuow5@mail2.sysu.edu.cn}

\author{Guang Tan}
\affiliation{%
  \institution{Sun Yat-sen University}
  \city{Shenzhen}
  \country{China}
}
\email{tanguang@mail.sysu.edu.cn}

\renewcommand{\shortauthors}{Trovato and Tobin, et al.}

\begin{abstract}
	Graph Neural Networks (GNNs) have received increasing attention for representation learning in various machine learning tasks. However, most existing GNNs applying neighborhood aggregation usually perform poorly on the graph with heterophily where adjacent nodes belong to different classes. In this paper we show that in typical heterophilous graphs, the edges may be directed, and whether to treat the edges as is or simply make them undirected greatly affects the performance of the GNN models. Furthermore, due to the limitation of heterophily, it is highly beneficial for the nodes to aggregate messages from similar nodes beyond local neighborhood. These motivate us to develop a model that adaptively learns the appropriate directionality for the graph, and exploits the underlying long-distance correlations between nodes. Here, we propose a spectral-based model, which enables the two conditions to be fused in a unified framework. Specifically, we first generalize the graph Laplacian to directed graph, termed {\em DiLaplacian}, based on the proposed Feature-aware PageRank, which simultaneously considers the graph directionality and long-distance correlation between nodes on the feature level. Based on DiLaplacian, we induce commute time between pair-wise nodes to further preserve the long-distance correlation on the topology level. Then the graph convolutional operator is induced by the commute time matrix. Extensive experiments on ten datasets with different levels of homophily demonstrate the effectiveness of our method.
	
\end{abstract}



\maketitle
\section{Introduction}\label{sec:intro}
GNNs have demonstrated remarkable performance in a wide spectrum of graph-based tasks, such as node classification~\cite{kipf2017semi, hamilton2017inductive,zhuo2023graph,zhuo2022proximity}, link prediction~\cite{zhang2018link}, graph classification~\cite{zhuo2022efficient}, and anomaly detection~\cite{zhuo2024partitioning}. Among the plentiful models that have been proposed, message passing neural networks (MPNNs)~\cite{gilmer2017neural} have been by far the most widely adopted. The main idea of MPNN-based GNNs is that each node aggregates messages from its neighbors and utilizes these messages to update its representation. 

Despite its success, such a recursively neighbor aggregating schema is susceptible to the quality of the graph. Specifically, MPNN-based GNNs rely on the assumption of homophily\footnote{Note that an easily-confounding concept that needs to be clarified is {\em (Homo)Heterophily $\neq$ (Homo)Heterogeneity} (See \Cref{app:notation}).}~\cite{zhu2020beyond} as a strong inductive bias, i.e., adjacent nodes tend to belong to the same class. However, homophily is not a universal principle, and there exist networks exhibiting heterophily, where 
a significant portion of edges connect nodes with different classes. Simply applying typical MPNN-based GNNs to a heterophilous graph would result in a huge drop in effectiveness~\cite{zhu2020beyond,chien2021adaptive}. For example, in social networks, the unique attributes of celebrities would be diluted if aggregated with the attributes of their followers. Hence, how to design a GNN architecture that works both for homophilous and heterophilous graphs has been a challenging topic recently. 



Recent works try to overcome the limitation of heterophily by stacking multiple GNN layers and combining intermediate representations~\cite{abu2019mixhop, chenWHDL2020gcnii,xu2018representation, yan2021two, chien2021adaptive,zhu2020beyond} to capture information from high-order neighborhoods. However, stacking layers may bring increasing parameters, causing the risks of overfitting, vanishing gradient, and oversmoothing~\cite{li2018deeper}. Such coarsely aggregating distant neighbors also results in introducing noise and irrelevant information~\cite{wang2021powerful}. 
Another class of methods aims to design adaptive frequency response filters~\cite{fagcn2021, dong2021graph} to integrate different signals in the process of message passing. While these methods can fuse positive information in per-hop neighbors, they still adopt the propagation on the original connectivity, such that inevitably affected by heterophily and can not capture sufficient information from intra-class nodes. Besides, some works about graph structure learning~\cite{franceschi2019learning, chen2020iterative,jiang2019semi} iteratively optimize the connectivity based on the learned embedding by parameterizing each edge locally. However, for semi-supervised tasks, the graph connectivity learned from labeled nodes is hard to generalize to unlabeled nodes, leading to `structural overfitting' for labeled nodes and exhibiting poor performance on test nodes for most classic heterophilous graph benchmarks~\cite{wang2021graph}.

We address the aforementioned issues of heterophily with a new GNN model, following the spectral-based approach. The model aims to capture underlying high-order (long-distance) correlations between nodes from both feature and topology perspectives in a preprocessing manner, avoiding stacking multiple layers and parameterizing the adjacency matrix. As another contribution of this study, our model has ability to adaptively learn the graphs' directionality, which is demonstrated to be the key to designing effective heterophilous GNNs~\cite{zhang2021magnet}. An example is shown in \Cref{tab:dir} that lists several common heterophilous graph benchmarks, all exhibiting a directed structure. The direction property of edges being used as is, or simply discarded, can greatly affect the message passing and the prediction accuracy. Overall, our proposed spectral-based model is a unified framework that considers directionality, high-order correlations on feature and topology levels simultaneously.

\begin{table}[t]
\centering
\caption{Node classification accuracy of GNNs in percent on directed graphs (subscript $d$) and their undirected versions (subscript $u$).}
\begin{adjustbox}{width=0.8\linewidth}
\begin{tabular}{cccccccc}
		\toprule
		&\textbf{Texas} & \textbf{Wisconsin} & \textbf{Cornell} & \textbf{Chameleon}\\
		\% Directed edges& 76.6& 77.9&86.9 &73.9\\  
		\midrule
		GCN$_u$         & 51.97 & 50.85        &  61.93           & 67.96                      \\
		GCN$_d$         & 61.08 & 53.14        &  57.84           & 63.82                      \\
		GAT$_u$         & 58.92 & 55.95        &  60.69           & 60.69                       \\
		GAT$_d$         & 52.73 & 61.37        &  62.73           & 57.66                       \\
		\bottomrule
\end{tabular}
\end{adjustbox}
\label{tab:dir}
\end{table}

Specifically, we first generalize the graph Laplacian defined for undirected graph to directed graph (digraph), called DiLaplacian (\textbf{Di}graph \textbf{Laplacian}), based on the transition matrix and its stationary distribution. To make the construction possible for general graphs and simultaneously capable of capturing long-distance feature similarity, we propose a Feature-aware PageRank algorithm to strengthen the original graph. The algorithm guarantees strong connectivity of the graph, and hence the existence of the stationary distribution; it also gives each node a certain probability to jump to other nodes with similar features. The graph convolutional operator based on the spectral analysis of DiLaplacian leads to a model called {\em DiL-GCN}, which captures the directionality and feature-level correlations between nodes.

On top of DiL-GCN, we further fuse long-distance correlations between nodes on the topology level, by defining a measure of node proximity based on pair-wise commute time. Commute time (CT) is a representative path-based feature that looks at the global graph structure. It provides a more strict and precise proximity measure between nodes in digraph than the shortest path. Specifically, the commute time between nodes $v_i$ and $v_j$ is the expected number of steps of a random walk to get from $v_i$ to $v_j$, and back to $v_i$. The less commute time from $v_i$ to $v_j$, the higher impact of $v_j$ on $v_i$, and therefore they should be more similar in the embedding space. Hence, we can use CT to measure the proximity between nodes in many real-world digraphs. For example, when performing a random walk on a social network, the walker starting at an ordinary user can immediately reach a celebrity it follows, but the walker can hardly return to the starting node by the celebrity's outgoing links, so the CT between the ordinary user and the celebrity may be very large and their proximity is low, even if they are adjacent. On the other hand, two ordinary users in the same compact community usually have fewer CT even if they are not connected directly, hence they have high proximity. In this paper, we theoretically prove that commute times between pair-wise nodes can be induced by DiLaplacian spectrum in a sparse manner, such that the derived commute time matrix can inherit the properties of DiLaplacian (i.e., feature-aware and directionality-dependent) while preserving long-distance correlations on the topology level. With such commute time matrix in hand, we can then define the graph convolution called {\em DiL-GCN$_{CT}$}. To the best of our knowledge, we are the first to integrate the commute time into the GNN model. 

We summarize our main contributions as follows:
\begin{itemize}
    \item We generalize graph Laplacian to digraph by proposing DiLaplacian. It preserves directionality and feature-level long-distance correlations.
    \item Based on DiLaplacian, we further induce commute time matrix, such that feature and topology level long-distance correlations, and directionality can be fused in a unified framework, called DiL-GCN$_{CT}$.
    \item We conduct comprehensive experiments on ten graph datasets, including seven heterophilous graphs and three homophilous graphs. The empirical evaluations demonstrate the superiority of our models.
\end{itemize}

\section{Related Work}
\subsection{Digraph Laplacian} \label{sec:rw_dl}
In contrast to undirected graphs not much is known about normalized Laplace operators for directed graphs. \citet{chung2005laplacians} defines a normalized Laplace operator for strongly connected directed graphs with nonnegative weights as $\mathbf{I}-\frac{\pi^{1 / 2} \mathbf{P} \pi^{-1 / 2}+\pi^{-1 / 2} \mathbf{P}^{*} \pi^{1 / 2}}{2}$. This Laplace operator is defined as a self-adjoint operator using the transition probability operator $P$ and the Perron vector $\pi$. \citet{singh2016graph} consider in-degree matrix and define the directed Laplacian as $\mathbf{D}_{\text {in }}-\mathbf{A}$, where $\mathbf{D}_{\text {in }}=\operatorname{diag}\left(\left\{d_{i}^{\text {in }}\right\}_{i=1}^{N}\right)$ is the in-degree matrix. \citet{li2012digraph} use stationary probabilities of the Markov chain governing random walks on digraphs to define Laplacian as $\pi^{\frac{1}{2}}(\mathbf{I}-\mathbf{P}) \pi^{-\frac{1}{2}}$. Different from prior works, we derive the directed graph Laplacian from the perspective of the graph signal and propose DiLaplacian $\Pi(\mathbf{D}^{-1}-\mathbf{P})$ (\cref{eq:dig2} in \Cref{sec:DiL}), which consists of degree matrix $\mathbf{D}$, transition matrix $\mathbf{P}$ and the Perron vector‘s diagonalization $\Pi$.

\subsection{GNN on Non-homophilous Graphs}
GNNs have achieved tremendous success on homophilous (assortative) graphs~\cite{kipf2017semi,hamilton2017inductive, wu2019simplifying, velickovic2018graph,qu2019gmnn,gilmer2017neural}. However, recent work~\cite{pei2019geom,zhu2020beyond} shows that traditional MPNN-based GNNs perform poorly on heterophilous (disassortative) graphs, and provide two metrics to measure the degree of homophily on the node level and edge level. To tackle the limitation of heterophily, the early method Geom-GCN~\cite{pei2019geom} precomputes unsupervised node embeddings and uses neighborhoods defined by geometric relationships in the resulting latent space to define graph convolution. H$_2$GCN~\cite{zhu2020beyond} proposes to make full use of high order neighborhoods, and combine self-embeddings and neighbor embeddings using concatenation. CPGNN~\cite{zhu2020graph} integrates the compatibility matrix as a set of learnable parameters into GNN, which it initializes with an estimated class compatibility matrix. PPNP~\cite{klicpera_predict_2019} and GDC~\cite{klicpera_diffusion_2019} redefine the node proximity using PageRank and Heat Kernel PageRank. GPRGNN~\cite{chien2021adaptive} performs feature aggregation for multiple steps to capture long-range information and then linearly combines the features aggregated with different steps, where the weights of the linear combination are learned during the model training. Some other methods like FAGCN~\cite{fagcn2021} use an attention mechanism and learn the weight of an edge as the difference in the proportion of low-frequency and high-frequency signals. AdaGNN~\cite{dong2021graph} leverages a trainable filter that spans across multiple layers to capture the varying importance of different frequency components for node representation learning. \citet{zhang2021magnet} experimentally show that for typical heterophilous graph benchmarks, the directionality of the graph plays an important role, and it is one of the keys to designing effective GNNs beyond homophily.

\section{Methodology}
\subsection{DiLaplacian}\label{sec:DiL}
In the general setting, $G= (V, E, \mathbf{X})$ is used to denote an unweighted directed graph with $N$ nodes, where $V= \{v_i\}^N_{i=1}$ is the node set, $E \subseteq(V \times V)$ is the edge set, $\mathbf{X} \in \mathbb{R}^{N \times d}$ is the node feature matrix with the number of features $d$ per node. Let $\mathbf{A} \in \mathbb{R}^{N \times N}$ be the adjacency matrix and $\mathbf{D} = diag(d_1,\cdots, d_N) \in \mathbb{R}^{N \times N}$ be the degree matrix of $\mathbf{A}$, where $d_i = \sum_{v_j \in V} \mathbf{A}(i,j) = \mathbf{A}\cdot e$ is the out-degree of $v_i$, and $e$ is an all-one vector. Let $\tilde{\mathbf{A}}= \mathbf{A}+ \mathbf{I}$ and $\tilde{\mathbf{D}} = \mathbf{D} + \mathbf{I}$ denote the augmented adjacency and degree matrix with self-loops, respectively. The transition probability matrix of the Markov chain associated with random walks on $G$ can be defined as $\mathbf{P} = \mathbf{D}^{-1}\mathbf{A}$, where $\mathbf{P}(i,j) = \mathbf{A}(i,j) / deg(v_i)$ is the probability of a 1-step random walk starting from $v_i$ to $v_j$. Graph Laplacian formulized as $\mathbf{L} = \mathbf{D} - \mathbf{A}$ is defined on the undirected graph whose adjacency matrix is symmetric. The augmented symmetrically normalized Laplacian with self-loop~\cite{wu2019simplifying} is defined as $\tilde{\mathcal{L}} = \tilde{\mathbf{D}}^{-\frac{1}{2}} \tilde{\mathbf{L}} \tilde{\mathbf{D}}^{-\frac{1}{2}}$, where $\tilde{\mathbf{L}} = \tilde{\mathbf{D}}-\tilde{\mathbf{A}}$. 


Some spectral-based GNNs~\cite{kipf2017semi,defferrard2016convolutional} simply make a directed graph be undirected by adding reverse edges to node pairs connected by single-directed edges using $\mathbf{A}_u = \frac{\mathbf{A} + \mathbf{A}^\top}{2}$. Although it helps explain GNNs in terms of spectral analysis, the original graph structure is disturbed due to the forced use of a symmetrized adjacency matrix, which can make different graphs share the same Laplacian and many important path-based features hard to obtain, such as hitting time and commute time. It is known that the graph Laplacian is defined as the divergence of the gradient of a signal on an undirected graph. For a signal $f \in \mathbb{R}^N$, $(\mathbf{L} f)(i)=\sum_{j \in \mathcal{N}_i} \mathbf{A}(i, j)(f(i)-f(j))$. In effect, the Laplacian on $f$ acts as a local averaging operator which is a node-wise measure of local smoothness. Following the meaning of the undirected graph Laplacian, we now generalize the existing spectral graph theory defined for undirected graphs to digraphs by defining Laplacian $\mathbf{L}$ acting on $f$: 
\begin{equation}
\begin{aligned}
(\mathbf{L}f)(i) &= \sum_{v_j \in \mathcal{N}^{+}_i} \mathbf{P}(i,j) (f(i)-f(j)) \\
&=((\mathbf{D}^{-1}-\mathbf{P})f)(i)
\end{aligned}
\label{eq:dig1}
\end{equation}
Following~\cite{chung2005laplacians,li2012digraph}, here we replace $\mathbf{A}$ with its row-normalization $\mathbf{P}$. Since edges are unweighted, then $\mathbf{P}(i,j) = \frac{1}{d_i}$. $\mathcal{N}^+_i$ is the set of $v_i$'s out-neighbors. The digraph Laplacian can be defined as $\mathbf{L} = \mathbf{D}^{-1}-\mathbf{P}$. In \cref{eq:dig1}, $\mathbf{L}$ acts as a local expectation operator of the difference of the signal between node and its neighbors. Nevertheless, $\mathbf{L}$ cannot directly capture the unique nature of random walks on the digraph. Specifically, a property of a random walk is that in the limit, the long-term average probability of being at a particular node is independent of the start node, or an initial probability distribution over nodes, provided only that the underlying graph is irreducible. This property provides a global perspective of node importance that $\mathbf{L}$ cannot capture. Leveraging the inherent equivalence between digraph and Markov chain~\cite{poole2014linear}, we can solve this problem by computing the stationary probability distribution and integrating it into $\mathbf{L}$, termed {\em DiLaplacian} $\mathbf{T}$.


Specifically, we firstly assume the digraph $G$ is irreducible and aperiodic (I\&A). A fundamental result from~\cite{poole2014linear} is that $G$ has a unique stationary probability distribution $\pi$ (i.e., Perron vector), satisfying the balance equation $\pi(i)=\sum_{v_j \in \mathcal{N}^-_i} \pi(j) \mathbf{P}(j,i)$, where $\mathcal{N}^-_i$ is the set of incoming neighbors of $v_i$. The stationary distribution $\pi$ can be computed by recurrence and it converges to the left eigenvector\footnote{Since the left eigenvector is a row vector, $\pi$ is its transpose.} of the dominant eigenvalue of the transition matrix $\mathbf{P}$, and $\pi$ satisfies $\sum_i \pi(i) = 1$ where the $i$-th element $\pi(i)$ is strictly positive. It can be interpreted as the limiting probability of finding a $\infty$ length random walk starting at any other nodes and ending at node $v_i$, i.e., $[\lim_{k \to \infty} \mathbf{P}^k](j,i) = \pi(i)$ where $j \in \{1,\cdots, N \}$. Thus, $\pi(i)$ can be used to measure the global importance of $v_i$, and we further integrate $\pi$ into $\mathbf{L}$ in \cref{eq:dig1} as:
\begin{equation}
\begin{aligned}
(\mathbf{T}f)(i) &= \sum_{v_j \in \mathcal{N}^{+}_i} \pi(i) \mathbf{P}(i,j) (f(i)-f(j)) \\
&=(\Pi(\mathbf{D}^{-1}-\mathbf{P})f)(i),
\end{aligned}
\label{eq:dig2}
\end{equation}
where $\Pi = diag(\pi(1),\cdots, \pi(n))$ and DiLaplacian $\mathbf{T} = \Pi(\mathbf{D}^{-1}-\mathbf{P})$. In form, the main differentia between DiLaplacian and other forms mentioned in \Cref{sec:rw_dl} is that DiLaplacian involves the degree matrix $\mathbf{D}$ and the Perron Vector $\pi$, which encode the local and global structure respectively.

However, our proposed DiLaplacian is based on a strong assumption, that is, $G$ is irreducible and aperiodic, which does not necessarily hold for general graphs. The input digraph may contain multiple connected components or absorbing nodes that make the graph reducible; or it may contain cyclic structures that make the graph periodic. In these cases, there is no guarantee for the positive, existence and uniqueness of the stationary distribution $\pi$. To make a digraph I\&A, ~\citet{tong2020digraph} add a teleporting probability distribution over all the nodes to address this issue, with the help of PageRank~\cite{page1999pagerank} that amends the transition matrix as $\mathbf{P}_{pr}=\gamma \mathbf{P} + (1-\gamma) \frac{e e^\top}{N}$, where $\gamma \in (0,1)$ and $e$ is the all-one column vector. $\mathbf{P}_{pr}$ means that the walker can randomly choose a non-neighbor node as the next step with probability $\frac{1-\gamma}{N}$. It is obvious that $\mathbf{P}_{pr}$ is irreducible and aperiodic, so it has a unique $\pi$. However, this solution yields a complete graph with a dense matrix $\mathbf{P}_{pr}$, which is extremely unfriendly to subsequent operations. Noticing that the input graph is attributed, i.e., every node has a feature vector $x \in \mathbb{R}^d$ (node degree, centrality or shortest path matrix can be node features if $\mathbf{X}$ is absent), we therefore provide an alternative method based on both features and topological structures of the graph, namely Feature-aware PageRank, which yields a sparse, irreducible and aperiodic transition matrix. In addition, for graphs with strong heterophily, the original topological structure is unreliable for MPNN-based GNNs, while the global feature-wise similarity between nodes provides an opportunity for the model to mitigate the effect of heterophilous structure.

\begin{figure*}[h!]
	\centering 
	\includegraphics[width=1\linewidth]{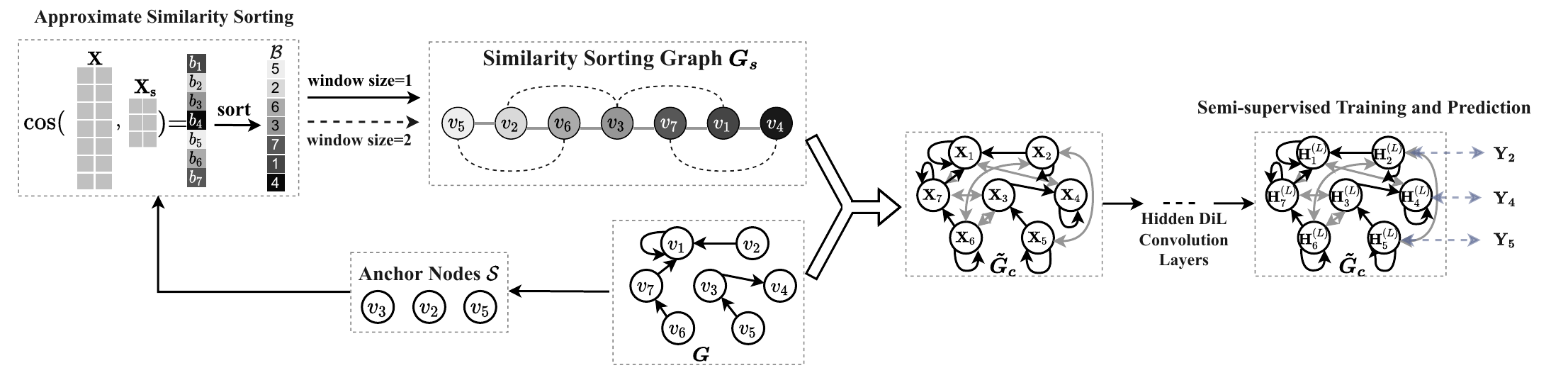}
	\caption{An illustration of the approximate similarity sorting. With the anchor node set $\mathcal{S}$, we compute the cosine similarity between all node features $\mathbf{X}$ and anchor node features $\mathbf{X}_s$ to obtain a similarity score vector. Then, we sort the similarity scores and return a sorted node index vector $\mathcal{B}$. We connect the nodes in order according to the score ranking in the $\mathcal{B}$ to generate a similarity sorting graph $G_s$. Note that $G_s$ is an undirected strongly connected graph, so $G_s$ is irreducible. At last, we combine $G_s$ and the input graph $G$ by adding all undirected edges in $G_s$ to $G$ (grey double-directional arrow), to generate an irreducible graph $G_c$.}
	\label{fig:fpr}
\end{figure*}

\subsection{Feature-aware PageRank}\label{sec:fpr}
Applying teleports in classical PageRank aims to solve two problems in the web graph, i.e., dead ends and spider traps~\cite{page1999pagerank}. Specifically, dead ends mean some nodes do not have any outgoing neighbors, which leads to PageRank scores ($\pi$ is the PR score vector) converging to 0 on all nodes. Spider traps mean there exist absorbing nodes in the graph, which leads to PageRank scores converging to 0 except the absorbing nodes. Thus, we only need to construct a \textbf{irreducible} graph, i.e., a strongly connected graph, that can solve the above two problems. Meanwhile, the graph need to be \textbf{aperiodic} to guarantee the PageRank transition matrix has a unique stationary distribution.

Instead of using $\mathbf{P}_{pr}$ as the transition matrix of the graph to build DiLaplacian, we propose Feature-aware PageRank (FPR), which aims to construct an irreducible graph based on node feature similarity, with its transition matrix denoted by $\mathbf{P}_{fpr}$. The main difference from $\mathbf{P}_{pr}$ is that $\mathbf{P}_{pr}$ gives teleport probabilities to all nodes, while $\mathbf{P}_{fpr}$ only gives teleport probabilities to $k$ nearest neighbors in the feature space. In other words, the walker at a node not only has a probability to transit to its outgoing neighbors, but is also subject to a probability of teleporting to $k$ nodes with most similar features to the current node, i.e., feature-aware random walk. 

A direct way is to construct a $k$NN graph based on $\mathbf{X}$, and then add $k$NN edges to $G$ such that $G$ satisfies feature-aware random walk. However, the $k$NN graph is not necessarily irreducible when $k$ is not large enough, such that combining it with $G$ may not yield an irreducible graph. Also, constructing $k$NN graph has $\mathcal{O}(N^2)$ time complexity. To address these issues, we propose the {\it approximate similarity sorting}, an anchor-based solution. The idea behind is the transitivity of similarity. Specifically, we first randomly sample a set of nodes as anchors to form an anchor set $\mathcal{S}=\{s_i\}$, where $|\mathcal{S}|$ is usually much smaller than $N$. For each node $v_i$, we compute the average similarity $b_i$ between node features of $v_i$ and every anchor node, as shown in \cref{eq:simsort}. Then we sort all nodes' average similarity values and return the indices of sorted nodes, denoted as $\mathcal{B} \in \mathbb{R}^N$.
\begin{equation}
    b_i=\frac{1}{|\mathcal{S}|} \sum^{|\mathcal{S}|}_{j=1} \cos ( v_i, s_j), \quad \mathcal{B} = \underset{i}{\mathrm{argsort}} \{b_i\}^N_{i=1},
    \label{eq:simsort}
\end{equation}
where $\cos(v_i,s_j)$ is the cosine similarity between node features of $v_i$ and an anchor node $s_j$. Then $\mathrm{argsort}(\cdot)$ returns the node indices that would sort $\{b_i\}_{i=1}^N$ in either descending or ascending order (different order will not affect the result). We next connect the nodes one by one based on the sorted node indices in $\mathcal{B}$ to construct a {\em similarity sorting graph} $G_s$ as shown in the left part of \Cref{fig:fpr}. For example, the similarity between $v_2$ and anchor nodes is close to that between $v_6$ and anchor nodes, i.e., the values of $b_2$ and $b_6$ are similar, then $v_2$ and $v_6$ are connected in $G_s$ by undirected edges. In other words, nodes with similar extents of similarity to anchor nodes are similar.

The similarity sorting graph $G_s = (\mathbf{A}_s, \mathbf{X})$ represents that each node approximatively connects to its top-$k$ most similar nodes, where $k = 2$ in \Cref{fig:fpr}. We can further add skip connection edges on $G_s$ and use the hyperparameter $window\_size$ to control the number of neighbors. $window\_size$ can be set to any positive integer, for example we show $window\_size = 2$, i.e., $k = 4$, by adding dummy line edges to $G_s$ as shown in \Cref{fig:fpr}. To guarantee sparsity, $window\_size$ is usually set to a small value. Overall, we construct $G_s$ to preserve the feature similarity while reduce the complexity to $\mathcal{O}(|\mathcal{S}|N)$. Besides, obviously $G_s$ is irreducible, which is one of the necessary conditions of DiLaplacian. By adding all edges of $G_s$ to $G$, we can get a combined graph $G_c = (\mathbf{A}_c, \mathbf{X})$, where $\mathbf{A}_c = \max(\mathbf{A}, \mathbf{A}_s)$, and $G_c$ is strictly irreducible. The transition probability matrix of $G_c$ is denoted as $\mathbf{P}_{fpr} = \mathbf{D}_c^{-1} \mathbf{A_c}$, where $\mathbf{D}_c$ is the out-degree matrix of $G_c$. Performing random walk on $G_c$ is called Feature-aware PageRank (FPR), which can be seen as giving the walker a certain probability of teleporting to other nodes who have similar features to the current node. Compared with PageRank transition matrix $\mathbf{P}_{pr}$, which uses all nodes as teleports, our FPR yields a sparse transition matrix $\mathbf{P}_{fpr}$ while ensuring irreducibility and it preserves the underlying high-order correlations in the feature space. 

Besides irreducibility, to guarantee the existence and uniqueness of the stationary distribution $\pi$, such that DiLaplacian can be constructed, the other necessary condition is that $\mathbf{P}_{fpr}$ should be aperiodic. To satisfy it, we only need to add a self-loop for each node with a non-zero probability, because the greatest common divisor of the lengths of its cycles is one. Thus, the transition matrix of FPR can be rewritten as $\tilde{\mathbf{P}}_{fpr} = \tilde{\mathbf{D}}_c^{-1} \tilde{\mathbf{A}}_c$, where $\tilde{\mathbf{A}}_c = \mathbf{A}_c + \mathbf{I}$ and $\tilde{\mathbf{D}}_c$ is 
the corresponding degree matrix. Therefore,  $\tilde{\mathbf{P}}_{fpr}$ is associated with an irreducible and aperiodic graph $\tilde{G}_c$. According to Perron-Frobenius Theory~\cite{poole2014linear}, $\tilde{\mathbf{P}}_{fpr}$ has a positive left eigenvector $\pi$ corresponding to its dominate eigenvalue 1 with algebraic multiplicity 1, which guarantees the random walk on $\tilde{G}_c$ converge to a unique positive stationary distribution equal to $\pi$. Therefore, the DiLaplacian in \cref{eq:dig2} can be rewritten as:
\begin{equation}
    \tilde{\mathbf{T}}=\Pi(\tilde{\mathbf{D}}_c^{-1} - \tilde{\mathbf{P}}_{fpr}).
    \label{eq:tilde_dil}
\end{equation}

\paragraph{Combinatorial Laplacian} Prior works~\cite{zhang2021magnet} and the results in \Cref{tab:dir} both demonstrate that the directionality, which describes whether to consider the directed structure, is uncertain for different datasets. Conventional MPNN-based GNNs mostly treat directionality as a binary hyperparameter, i.e, treat the edges as is or simply make them undirected, and the choice of directionality depends on performance on the validation set under multiple trials. Here, we aim to do a `soft' treatment of the graph’s directionality, such that the model can adaptively learn the appropriate directionality. Specifically, we use $G_{c_u} = (\mathbf{A}_{c_u},\mathbf{X})$ to represent the undirected version of $G_c$, where $\mathbf{A}_{c_u} = \max (\mathbf{A}_c,\mathbf{A}_c^\top)$ is a symmetric adjacency matrix of $G_{c_u}$. Then we combine DiLaplacian defined in \cref{eq:tilde_dil} and undirected graph Laplacian to define the combinatorial Laplacian $\mathbf{L}_c$ as:
\begin{equation}
    \mathbf{L}_{comb} = \alpha \mathbf{L} + \beta \tilde{\mathbf{T}}, 
    \label{eq:comb_lap}
\end{equation}
where $\mathbf{L} = \mathbf{D}_{c_u} - \mathbf{A}_{c_u}$ and $\mathbf{D}_{c_u}$ is the degree matrix of $G_{c_u}$. $\alpha, \beta \in \mathbb{R}^1$ are learnable parameters to adaptively adjust the bias of directionality. 

\subsection{DiL-GCN}
In this subsection, we derive a graph convolutional operator based on the combinatorial Laplacian $\mathbf{L}_c$ in \cref{eq:comb_lap}. For an I\&A undirected graph, its Perron vector $\pi$ of transition matrix can be computed by $\pi(i) = \frac{d_i}{\sum_k d_k}$, so the vector of node degrees $\mathbf{A} \cdot e$ is a scalar multiple of $\pi$. On the other hand, for an I\&A digraph, $\pi(i)=\sum_{v_j \in \mathcal{N}^-_i} \pi(j) \mathbf{P}_{fpr}(j,i)$ is the sum of all incoming probabilities from the neighbors of $v_i$, it hence plays the same role of degree vector in undirected graph that reflects the connectivity between nodes~\cite{gleich2006hierarchical}. Based on these properties, \cite{chung2005laplacians} and \cite{zhou2005learning} define the normalized digraph Laplacian as a Hermitian matrix by using $\Pi$ to normalize transition matrix. Following the same spirit, we define the normalized DiLaplacian for $\tilde{\mathbf{T}}$ as:
\begin{equation}
    \mathcal{T} = \frac{1}{2} \left(\Pi^{\frac{1}{2}} (\tilde{\mathbf{D}}_{c}^{-1} - \tilde{\mathbf{P}}_{fpr}) \Pi^{-\frac{1}{2}} + 
    \Pi^{-\frac{1}{2}} (\tilde{\mathbf{D}}_{c}^{-1} - \tilde{\mathbf{P}}_{fpr}^\top) \Pi^{\frac{1}{2}} \right),
\label{eq:norm_dig}
\end{equation}
which is a real-valued matrix with a full set of real eigenvalues. It is known that the normalized $\mathbf{L}$ can be defined as $\mathcal{L} = \mathbf{I} - \mathbf{D}_{c_u}^{-\frac{1}{2}}\mathbf{A}_{c_u}\mathbf{D}_{c_u}^{-\frac{1}{2}}$. Then the normalized combinatorial Laplacian can be represented as: 
\begin{equation}
    \mathcal{L}_{comb} = \alpha \mathcal{L} + \beta \mathcal{T}
\end{equation}

Recall that GCN~\cite{kipf2017semi} uses the $1$-st order Chebyshev polynomials with specific coefficients to approximate a graph filter as $\widehat{\mathbf{A}}_{c_u} = \mathbf{I}-\mathcal{L} = \mathbf{D}_{c_u}^{-\frac{1}{2}}\mathbf{A}_{c_u}\mathbf{D}_{c_u}^{-\frac{1}{2}}$, which serves as a convolution support to do neighbor weighted aggregation for each node. Equipped with the normalized combinatorial Laplacian $\mathcal{L}_{comb}$ and node information matrix $H^{(\ell)}$ at iteration $\ell$, we use the following graph convolution layer to define DiL-GCN: 
\begin{equation}
    H^{(\ell+1)} = \sigma(H^{(\ell)}W_0^{(\ell+1)} + (\alpha \widehat{\mathbf{A}}_{c_u} + \beta \widehat{\mathcal{T}})H^{(\ell)}W_1^{(\ell+1)})
    \label{eq:digcn}
\end{equation}
where $\sigma$ is the activation function, $W_0$ and $W_1$ the trainable matrix, $\widehat{\mathbf{A}}_{c_u}$ the convolutional operator corresponding to $\mathcal{L}$, $H^{(0)} = \mathbf{X}$. In addition, we separate the self-embedding to preserve more personalized information, which is an effective design for networks with heterophily~\cite{zhu2020beyond}. $\widehat{\mathcal{T}}$ is the other convolutional operator derived by the normalized DiLaplacian $\mathcal{T}$:
\begin{equation}
    \widehat{\mathcal{T}} = \tilde{\mathbf{D}}_{c}^{-1} + \frac{1}{2} \left(\Pi^{\frac{1}{2}} \tilde{\mathbf{P}}_{fpr} \Pi^{-\frac{1}{2}} + \Pi^{-\frac{1}{2}}  \tilde{\mathbf{P}}_{fpr}^\top \Pi^{\frac{1}{2}}\right).
    \label{eq:aug_DiLaplacian}
\end{equation}
The detailed derivation procedure from $\mathcal{T}$ to $\widehat{\mathcal{T}}$ is based on spectral analysis, which is available in \Cref{app:spectral}. 

Also of note, the two parts of $\mathbf{L}_{comb}$ are both built based on the combined graph $G_c$, thus DiL-GCN is able to capture the long-distance correlations on the feature level. On the other hand, from the perspective of topology level, graph topology also provides useful information to guide the node to detect underlying related nodes beyond local neighborhoods, which requires a measure of node proximity that is able to look at the global graph structure. To address this problem, we further apply commute time to estimate the proximity in topology space and propose DiL-GCN$_{CT}$, which can be induced by our proposed DiLaplacian, hence it essentially inherit the properties of the directed structure and long-distance feature correlations.



\subsection{DiLaplacian with Commute Time}
Based on the FPR transition matrix, we can preserve the underlying long-distance correlation on the feature level when performing a random walk on the combined graph $G_c$. Further, we aim to explore long-distance correlation on the topology level, such that these two aspects can be fused in a unified framework. Previous works~\cite{klicpera_predict_2019, wu2019simplifying, klicpera_diffusion_2019} have considered node proximity to reflect length of the shortest path between nodes, that is, the fewer hops from $v_i$ to $v_j$, the higher proximity of $v_j$ to $v_i$. This however does not fully reflect the mutual relationship between nodes in the network topology, especially in digraphs. Given the direct edge (or hierarchy), the distance between nodes is not only related with the steps between them, but also with the depth hierarchy of respective nodes. Thus, we define a measure of node proximity based on commute time.

For two nodes $v_i$ and $v_j$, the hitting time $\mathcal{H}(i,j)$ is the expected steps it takes for a random walk to travel from $v_i$ to $v_j$, and the commute time between them $\mathcal{C}(i,j) = \mathcal{H}(i,j) + \mathcal{H}(j,i)$ is defined as the expected time it takes a random walk from $v_i$ to $v_j$ and then back to $v_i$. The significance lies in that if $v_i$ has a high probability of returning to itself via $v_j$ in a random walk, then $v_i$ and $v_j$ have high proximity. Since dissimilarity between adjacent nodes leads to the heterophily of the graph, we hence impose a stronger restriction on node proximity based on the commute time, such that nodes can sense distant nodes. To preserve such proximity in the GNN model, we propose DiL-GCN$_{CT}$, which defines the graph propagation matrix based on commute time. 

By the results from the standard Markov chain theory~\cite{aldous2002reversible}, we can calculate expected hitting times for random walks on any graph in terms of the fundamental matrix. Given an I\&A digraph $G_c$ with irreducible transition matrix $\tilde{\mathbf{P}}_{fpr}$ and its stationary distribution $\pi$, the fundamental matrix $\mathbf{Z}$ is defined as $\mathbf{Z} = \sum^{\infty}_{t = 0} (\tilde{\mathbf{P}}_{fpr}^t - e \pi^\top)$. It is proved~\cite{li2012digraph} that $\mathbf{Z}$ converges to:
\begin{equation}
    \mathbf{Z} = (\mathbf{I}-\tilde{\mathbf{P}}_{fpr} +\mathbf{J} \Pi)^{-1} - \mathbf{J}\Pi,
    \label{eq:fm}
\end{equation}
where $\mathbf{J} = ee^\top$ is an all-one matrix. Then the hitting time from $v_i$ to $v_j$ has the expression $\mathcal{H}(i,j) = \frac{\mathbf{Z}(j,j)-\mathbf{Z}(i,j)}{\pi(j)}$. We can derive its matrix form:
\begin{equation}
    \mathcal{H} = (\mathbf{J} diag(\mathbf{Z})-\mathbf{Z}) \Pi^{-1},
    \label{eq:hitting}
\end{equation}
where $diag(\mathbf{Z})$ is the diagonal matrix formed by the main diagonal of $\mathbf{Z}$. In \cref{eq:fm}, the first term involves the inversion of a dense matrix, which has $\mathcal{O}(N^3)$ time complexity. Based on our proposed DiLaplacian $\tilde{\mathbf{T}}$ in \cref{eq:tilde_dil}, we can compute $\mathbf{Z}$ in a more efficient way. 

\begin{theorem}
Given a combined graph $\tilde{G}_c$ and its Feature-Aware PageRank transition matrix $\tilde{\mathbf{P}}_{fpr}$, the DiLaplacian of $\tilde{G}_c$ is defined as $\tilde{\mathbf{T}}=\Pi(\tilde{\mathbf{D}}_c^{-1} - \tilde{\mathbf{P}}_{fpr})$. Then the fundamental matrix $\mathbf{Z}$ of $\tilde{G}_c$ can be solved by:
\begin{equation}
    \mathbf{Z} = \Pi^{-\frac{1}{2}} (\Pi^{-\frac{1}{2}} \tilde{\mathbf{T}} \Pi^{-\frac{1}{2}} - \tilde{\mathbf{D}}_{c}^{-1} + \mathbf{I})^{\dagger} \Pi^{\frac{1}{2}},
    \label{eq:fm0}
\end{equation}
where the superscript $\dagger$ means Moore–Penrose pseudoinverse of the matrix.
\label{the:Z}
\end{theorem}

The proof is given in \Cref{app:proof}. The pseudoinverse of the sparse matrix in \cref{eq:fm0} can be calculated via low-rank SVD. In particular, we use ARPACK~\cite{lehoucq1998arpack}, an iteration method based on the restarted Lanczos algorithm, as an eigensolver. ARPACK depends on matrix-vector multiplication. Usually a small number of iterations is enough, so if the matrix is sparse and the matrix-vector multiplication can be done in $\mathcal{O}(N)$ time parallelly, then the eigenvalues are found in $\mathcal{O}(N)$ time as well.

Based on the fundamental matrix defined in \cref{eq:fm0}, we can compute the hitting time matrix $\mathcal{H}$ by \cref{eq:hitting}. The commute time matrix can be obtained by $\mathcal{C} = \mathcal{H} + \mathcal{H}^\top$, and we set $\mathcal{C}(i,i) = 0$ for $i = 1,\cdots,N$. Each entry $\mathcal{C}(i,j)$ is the commute time between $v_i$ and $v_j$.
However, since $\tilde{G}_c$ is irreducible, all entries in $\mathcal{C}$ is positive, which yields a fully connected graph. It makes the subsequent propagation step computationally expensive and makes little sense for most downstream tasks. Hence, we sparsify $\mathcal{C}$ with a certain ratio $\mu \in (0,1)$. For each row of $\mathcal{C}$, we set $\mu N$ largest entries to 0. In doing so, we obtain a sparse non-negative matrix $\mathcal{C}_s$. Then, we define a graph convolutional operator based on $\mathcal{C}_s$ as:
\begin{equation}
\widehat{\mathcal{C}}_s = \mathbf{D}_s^{-1}\exp(-\mathcal{C}_s).
\label{eq:norm_cummute}
\end{equation}
where $\mathbf{D}_s$ a diagonal matrix, $\mathbf{D}_s(i,i) = \sum_j \exp(-\mathcal{C}_s)(i,j)$. Since the commute time matrix $\mathcal{C}$ is based on the DiLaplacian $\tilde{\mathbf{T}}$ of the combined graph $\tilde{G}_c$, $\widehat{\mathcal{C}}_s$ hence inherit properties of $\tilde{\mathbf{T}}$. In this way, we fuse the directed information and long-distance correlations between nodes on the feature and topology level into the convolutional operator $\widehat{\mathcal{C}}_s$.
Then we replace $\widehat{\mathcal{T}}$ in \cref{eq:digcn} with $\widehat{\mathcal{C}}_s$ and define the $\ell$-th layer of DiL-GCN$_{CT}$ as:
\begin{equation}
    H^{(\ell+1)} = \sigma(H^{(\ell)}W_0^{(\ell+1)} + (\alpha \widehat{\mathbf{A}}_{c_u} + \beta \widehat{\mathcal{C}}_s)H^{(\ell)}W_1^{(\ell+1)}).
\end{equation}


{\bf Prediction.} The class prediction $\hat{\mathbf{Y}}$ of a $L$-layer DiL-GCN or DiL-GCN$_{CT}$ is based on node embeddings in the last layer:
\begin{equation}
    \hat{\mathbf{Y}} = \mathrm{softmax}(H^{(L)}),
\end{equation}
where $W_0^{(L)}, W_1^{(L)} \in \mathbb{R}^{d^{(L-1)} \times m}$, $d^{(L-1)}$ the dimension of the $(L-1)$-th layer and $m$ the number of classes. Since our DiLaplacian and the commute time matrix can capture the underlying global correlations, the number of layers $L$ hence set to a small value.

\subsection{Time Complexity}
In the stage of constructing the combined graph $G_c$, the time complexity of the approximate similarity sorting including cosine similarity computation $\mathcal{O}(N)$ and similarity sorting $\mathcal{O}(N\log N)$. Then we compute the stationary distribution $\pi$ using power iteration, whose time complexity is $\mathcal{O}(tN)$, where $t \ll N$ is usually small. The total time complexity of an $L$-layer DiL-GCN is therefore $\mathcal{O}(N \log N + LNd^2)$. Further, the time complexity of computing the hitting matrix by the fundamental matrix based on \cref{eq:fm0} is $\mathcal{O}(N)$. The time complexity of row-wise sparsification of $\widehat{\mathcal{C}}_s$ is $\mathcal{O}(N \log N)$. In total, the DiL-GCN$_{CT}$ has the same time complexity as DiL-GCN.

\section{Experiments}
\subsection{Datasets}
We conduct experiments\footnote{Code available at \href{https://anonymous.4open.science/r/DiLGCN-CT}{https://anonymous.4open.science/r/DiLGCN-CT}.} on seven disassortative (heterophilous) graph datasets and three assortative (homophilous) graph datasets, which are widely used in previous work~\cite{pei2019geom,zhu2020beyond,chien2021adaptive}. Specifically, the disassortative datasets including three web page graphs from the WebKB dataset~\cite{pei2019geom} ({\it Texas}, {\it Wisconsin} and {\it Cornell}), three webpage graphs from Wikipedia~\cite{pei2019geom} ({\it Actor}, {\it Chameleon} and {\it Squirrel}) and a social network of European Deezer~\cite{lim2021new} ({\it deezer}). On the other hand, the assortative datasets including two citation networks ({\it CoraML}~\cite{bojchevski2018deep} and {\it Citeseer}~\cite{frasca2020sign}) and a coauthor network ({\it CoauthorCS}~\cite{shchur2018pitfalls}). In order to distinguish assortative and disassortative graph datasets, \citet{zhu2020beyond} propose the edge homophily ratio as a metric to measure the homophily of a graph $h = \frac{|\{ (v_i, v_j) \in E: y_i = y_j \}|}{|E|}$,
where $y_i$ is the label of $v_i$. This metric is defined as the proportion of edges that connect two nodes of the same class. The datasets that we used have edge homophily ratio ranging from low to high. 

For all disassortative datasets except deezer, we use the feature vectors, class labels, and 10 fixed splits (48\%/32\%/20\% of nodes per class for train/validation/test) from~\cite{pei2019geom}. For deezer, we use 5 fixed splits (50\%/25\%/25\% for train/validation/test) provided by~\cite{lim2021new}. For all assortative datasets, we use the same split as~\cite{tong2020digraph}, i.e., 20 labels per class for the training set, 500 labels for validation set and the rest for test set. The detailed information and statistics of these datasets are shown in \Cref{tab:datasets}.

\begin{table*}[ht!]
    \centering
    \caption{%
  Classification accuracy (with standard deviation) in percent. The "$*$" represents the best results among all variants of the model. For all baselines except MLP and ChebyNet, we respectively conduct experiments on symmetrized and original adjacency matrices and report the better results. \colorbox{gray!15}{\textbf{Bold}}: best; \underline{Underline}: runner-up.
    }
    \label{tab:nc_result}
    \begin{adjustbox}{width=1\textwidth}
    \begin{tabular}{ccccccccccc} %
    \toprule
            &  \texttt{\bf Texas}           &   \texttt{\bf Wisconsin}           &   \texttt{\bf Actor}           
            &   \texttt{\bf Squirrel}       &   \texttt{\bf Chameleon}           & \texttt{\bf Cornell} 
            & \texttt{\bf deezer}          &  \texttt{\bf Citeseer}       &   \texttt{\bf CoraML}            &   \texttt{\bf CoauthorCS}  \\
          \textbf{$h$} & \textbf{0.11} & \textbf{0.21} & \textbf{0.22} & \textbf{0.22} & \textbf{0.23} & \textbf{0.3} & \textbf{0.53}&  \textbf{0.74} & \textbf{0.79} & \textbf{0.81} \\
    \midrule
    MLP   & 81.89($\pm$4.78) & 85.29($\pm$3.31) & 35.76($\pm$0.98) & 31.68($\pm$1.90) &46.21($\pm$2.99) &81.89($\pm$6.40) &  66.55($\pm$0.72)     & 65.41($\pm$1.74) & 77.12($\pm$0.96)  & 83.01($\pm$1.03)  \\ 
    \midrule
    GCN   & 61.08($\pm$6.07) & 53.14($\pm$5.29) & 30.26($\pm$0.79) & 52.43($\pm$2.01) & 67.96($\pm$1.82) &61.93($\pm$3.67)  &  62.23($\pm$0.53) & 66.03($\pm$1.88) &  81.18($\pm$1.25) & 91.33($\pm$0.45)\\
    GAT+JK & 71.08($\pm$5.41) & 72.16($\pm$3.37) & 33.72($\pm$0.95) & 46.57($\pm$1.77)& 58.35($\pm$1.89) &70.54($\pm$4.43) & 60.37($\pm$2.01) & 65.33($\pm$0.97) & 68.31($\pm$2.12) & 87.42($\pm$0.57)\\
    GCN+JK & 69.19($\pm$6.53) & 74.31($\pm$6.43) & 34.18($\pm$0.85) & 40.45($\pm$1.61) & 63.42($\pm$2.00) &64.59($\pm$8.68)&  60.99($\pm$0.14)& 64.20($\pm$1.81)& 80.16($\pm$0.87) & 89.60($\pm$0.22) \\
    ChebyNet & 79.22($\pm$7.51) & 81.63($\pm$6.31)  &28.27($\pm$1.99) & 32.51($\pm$1.46) & 59.17($\pm$1.89) &79.84($\pm$5.03)& 67.02($\pm$0.59) &66.71($\pm$1.64)& 80.03($\pm$1.82) & 91.20($\pm$0.40)\\
    GAT   & 58.92($\pm$5.10) & 61.37($\pm$5.27) & 26.28($\pm$1.73) & 40.72($\pm$1.55) & 60.69($\pm$1.95) & 62.73($\pm$3.68) &    
    61.09($\pm$0.77)& 67.58($\pm$1.39) &80.41($\pm$1.77) & 90.82($\pm$0.48) \\
    GraphSAGE & 83.92($\pm$6.14) & 85.07($\pm$5.56) & 34.23($\pm$0.99) & 41.61($\pm$0.74) & 58.73($\pm$1.68) &  80.09($\pm$6.29) & 64.28($\pm$1.13)
    &66.81($\pm$1.38) &80.03($\pm$1.70) & 90.15($\pm$1.03) \\
    APPNP & 79.57($\pm$5.32) & 81.29($\pm$2.57) & \underline{35.93}($\pm$1.04) & 51.91($\pm$0.56) & 45.37($\pm$1.62) & 70.96($\pm$8.66) &  \underline{67.21}($\pm$0.56) & 66.90($\pm$1.82) & 81.31($\pm$1.47) & 87.24($\pm$1.06)\\
    AM-GCN & 74.57($\pm$6.45) & 76.42($\pm$5.62) & 32.47($\pm$0.66) & 46.88($\pm$1.30) & 60.10($\pm$2.57) & 79.63($\pm$8.10) & 64.28($\pm$0.89) & 67.89($\pm$2.30) & 81.97($\pm$1.62) & 91.73($\pm$0.65)\\
    \midrule
    MixHop & 77.84($\pm$7.73) & 75.88($\pm$4.90) & 32.22($\pm$2.34) &43.80($\pm$1.48)& 60.50($\pm$2.53) & 73.51($\pm$6.34)& 66.80($\pm$0.58)& 56.09($\pm$2.08)  & 65.89($\pm$1.50) &  88.97($\pm$1.04) \\
    FAGCN & 82.43($\pm$6.89) & 82.94($\pm$7.95)& 34.87($\pm$1.25)  &42.59($\pm$0.79) &55.22($\pm$3.19) &79.19($\pm$9.79) & 65.88($\pm$0.31)&  \underline{68.93}($\pm$1.17)  & \cellcolor{gray!15}\textbf{84.00}($\pm$1.05) & 91.07($\pm$1.28)\\
    H$_2$GCN-1 & 84.86($\pm$6.77) & 86.67($\pm$4.69) & 35.86($\pm$1.03) & 36.42($\pm$1.89) & 57.11($\pm$1.58) & 82.16($\pm$4.80) & \cellcolor{gray!15}\textbf{67.49}($\pm$1.18) &  64.57($\pm$2.06)    & 80.66($\pm$0.97) & 88.45($\pm$0.97) \\
    H$_2$GCN-2 & 82.16($\pm$5.28) & 85.88($\pm$4.22) & 35.62($\pm$1.30) & 37.90($\pm$2.02) & 59.39($\pm$1.98) & 82.16($\pm$6.00) &  65.04($\pm$0.73)     & 67.15($\pm$0.99) & 78.33($\pm$1.29)&88.53($\pm$0.38) \\
    CPGNN* & 82.63($\pm$6.88) & 84.58($\pm$2.72) & 35.76($\pm$0.92) & 29.25($\pm$4.17) & 65.17($\pm$3.17) & 79.93($\pm$6.12) &  58.26($\pm$0.71)     &   66.19($\pm$1.74)    & 81.02($\pm$0.77) &89.20($\pm$0.91) \\
    GPRGNN & 84.43($\pm$4.10) & 83.73($\pm$4.02) & 33.94($\pm$0.95) & 50.56($\pm$1.51) & 66.31($\pm$2.05) & 79.27($\pm$6.03) &  66.90($\pm$0.50) & 61.74($\pm$1.87) & 73.31($\pm$1.37) & 91.49($\pm$0.39) \\
    GCNII & 77.57($\pm$3.83) & 80.39($\pm$3.40) & 34.52($\pm$1.23) & 38.47($\pm$1.58) & 63.86($\pm$3.04) & 77.86($\pm$3.79) &   66.18($\pm$0.93)    & 58.32($\pm$1.93) & 64.72($\pm$2.85)&84.13($\pm$1.91) \\
    \midrule
    DGCN  & 71.53($\pm$7.22) & 65.52($\pm$4.71) & 33.74($\pm$0.25) & 37.16($\pm$1.72) & 50.77($\pm$3.31) & 68.32($\pm$4.30) & 62.11($\pm$2.14) & 66.37($\pm$1.93) &78.36($\pm$1.41) &88.41($\pm$0.68)   \\
    DiGCN & 65.18($\pm$8.09) & 60.06($\pm$3.82) & 32.45($\pm$0.78) & 34.76($\pm$1.24) & 50.55($\pm$3.38) &67.75($\pm$6.11)& 58.37($\pm$0.89)& 63.77($\pm$2.27) & 79.51($\pm$1.34)  & OOM \\
    DiGCN-IB & 66.97($\pm$13.72) & 64.19($\pm$7.01) & 32.82($\pm$0.68) & 33.44($\pm$2.07) & 50.37($\pm$4.31) & 65.01($\pm$10.33)& 55.39($\pm$2.88) & 64.99($\pm$1.72) & 81.07($\pm$1.14) & 91.09($\pm$0.32) \\
    \midrule
    {\bf DiL-GCN}& \cellcolor{gray!15}\textbf{87.84}($\pm$3.25) & \underline{86.86}($\pm$4.21)& 35.57($\pm$1.10) & \underline{54.58}($\pm$2.30) & \underline{70.70}($\pm$2.07) & \underline{85.41}($\pm$5.56) & 66.52($\pm$0.17)  & \cellcolor{gray!15}\textbf{68.97}($\pm$1.38)& \underline{82.30}($\pm$0.81)  &   \underline{92.06}($\pm$0.79)\\
    {\bf DiL-GCN$_{CT}$} &   \underline{86.52}($\pm$3.66)    &  \cellcolor{gray!15}\textbf{87.06}($\pm$4.03)     &   \cellcolor{gray!15}\textbf{36.10}($\pm$1.02) &\cellcolor{gray!15} \textbf{56.22}($\pm$2.09)  & \cellcolor{gray!15}\textbf{72.19}($\pm$1.25) &  \cellcolor{gray!15}\textbf{87.30}($\pm$7.18) & 66.90($\pm$0.72) & 67.85($\pm$1.68) & 81.82($\pm$0.64) & \cellcolor{gray!15}\textbf{92.15}($\pm$0.52)\\
	\bottomrule
    \end{tabular}
    \end{adjustbox}
\end{table*}

\subsection{Baselines}
We compare our DiL-GCN and DiL-GCN$_{CT}$ against 19 baselines. The full list of methods are:
\begin{itemize}
\item Structure-independent: 2-layer MLP.
\item General GNNs: GCN~\cite{kipf2017semi}, ChebyNet~\cite{defferrard2016convolutional}, GAT~\cite{velickovic2018graph}, GraphSAGE~\cite{hamilton2017inductive}, APPNP~\cite{klicpera_predict_2019}, and jumping knowledge networks (GCN+JK, GAT+JK)~\cite{xu2018representation}.
\item Digraph GNNs: DGCN~\cite{tong2020directed}, DiGCN and DiGCN-IB~\cite{tong2020digraph}.
\item Feature similarity preserving GNN: AM-GCN~\cite{wang2020gcn}.
\item Non-homophilous GNNs: MixHop~\cite{abu2019mixhop}, H$_2$GCN-1/2~\cite{zhu2020beyond}, FA\\GCN~\cite{fagcn2021}, CPGNN~\cite{zhu2020graph}, GPRGNN~\cite{chien2021adaptive}, and GCNII~\cite{chenWHDL2020gcnii}.
\end{itemize}
For all baselines, we report their performance based on their official implementations after careful hyperparameter tuning. 

\begin{table}[h!]
\begin{threeparttable}
	\small
	\centering
	\caption{Statistics of the datasets.}
	\label{tab:datasets}
    \centering
		\begin{tabular}{ccccccccc} 
			\toprule			
			Dataset & $N$ & $|E|$ & \# Feat.&  \# Classes & Digraph\\
			\midrule
			Texas\tnotex{fn:texas}           & 183     & 309    & 1,703 & 5        &\checkmark   \\
			Wisconsin\tnotex{fn:texas}       & 251     & 499    & 1,703 & 5       &\checkmark  \\
			Cornell\tnotex{fn:texas}         & 183     & 295    & 1,703 & 5       &\checkmark\\
			Chameleon\tnotex{fn:texas}       & 2,277    & 36,101  & 2,325 & 5      &\checkmark  \\
			Squirrel\tnotex{fn:texas}        & 5,201    & 217,073 & 2,089 & 5    &\checkmark \\
			Actor\tnotex{fn:texas}           & 7,600    & 33,544  & 931  & 5        &\checkmark\\
			deezer\tnotex{fn:deezer}          & 28,281   & 92,752  &31,241 & 2       \\
			Cora-ML\tnotex{fn:coraml}         & 2,995    & 8,416   & 2,879  & 7      &\checkmark \\
			Citeseer\tnotex{fn:citeseer}        & 3,312    & 4,715   & 3,703  & 6       &\checkmark\\
			CoauthorCS\tnotex{fn:coauthor-cs}      & 18,333   & 81,894  & 6,805  & 15      \\
			\bottomrule
	\end{tabular}
		\begin{tablenotes}[flushleft]
	\scriptsize{
		\item[1] \label{fn:texas} \url{https://github.com/graphdml-uiuc-jlu/geom-gcn/tree/master/splits}
		\item[2] \label{fn:deezer} \url{https://github.com/CUAI/Non-Homophily-Benchmarks/blob/main/data/splits/deezer-europe-splits.npy}
		\item[3] \label{fn:coraml} \url{https://github.com/abojchevski/graph2gauss/raw/master/data/cora_ml.npz}
		\item[4] \label{fn:citeseer} \url{https://github.com/abojchevski/graph2gauss/raw/master/data/citeseer.npz}
		\item[5] \label{fn:coauthor-cs} \url{https://github.com/shchur/gnn-benchmark/raw/master/data/npz/ms_academic_cs.npz}

	}
	\end{tablenotes}
	\end{threeparttable}
	\addtocounter{footnote}{+5}
\end{table}

\subsection{Experimental Setting}\label{sec:exp_setting}
For undirected graphs (deezer and CoauthorCS), i.e., all edges are bidirectional in the provided raw data, we directly use the original adjacency matrix in all baselines. In the experiments of digraph datasets, for ChebyNet as a spectral method, we use the symmetrized adjacency matrix. For other baselines, we apply both the symmetrized and asymmetric adjacency matrix for node classification. The results reported are the better of the two results. Note that GCN is a spectral method, but it can be interpreted from the spatial perspective, i.e., outgoing neighbor aggregation with specific weights $\frac{1}{\sqrt{d_i d_j}}$. Hence, we view GCN as a spatial method. 
Other implementation details including running environment, hyperparameter settings and search space are presented in \Cref{app:id}. 

\begin{figure*}[!tb]
\centering
\subfloat[Citeseer]{{\includegraphics[width=0.31\linewidth]{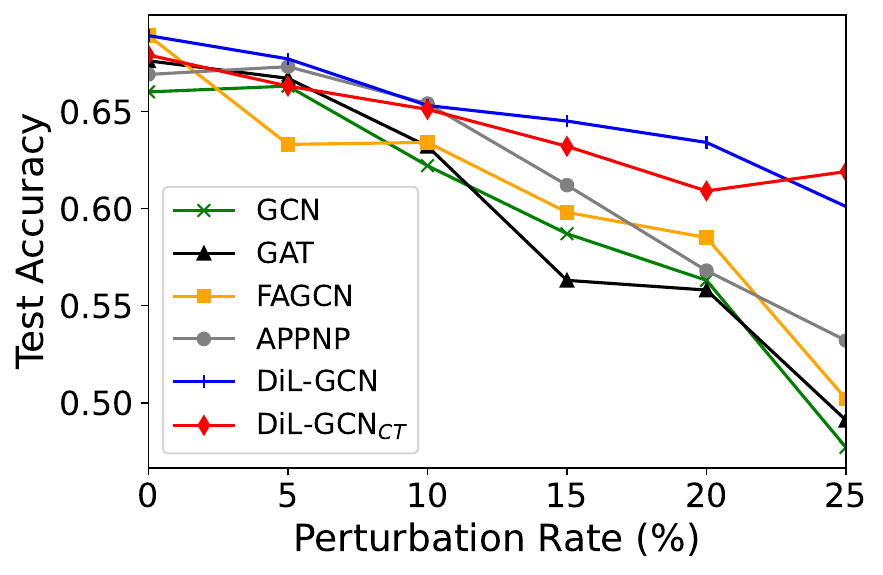} }\label{fig:att_citeseer}}
\subfloat[CoraML]{{\includegraphics[width=0.31\linewidth]{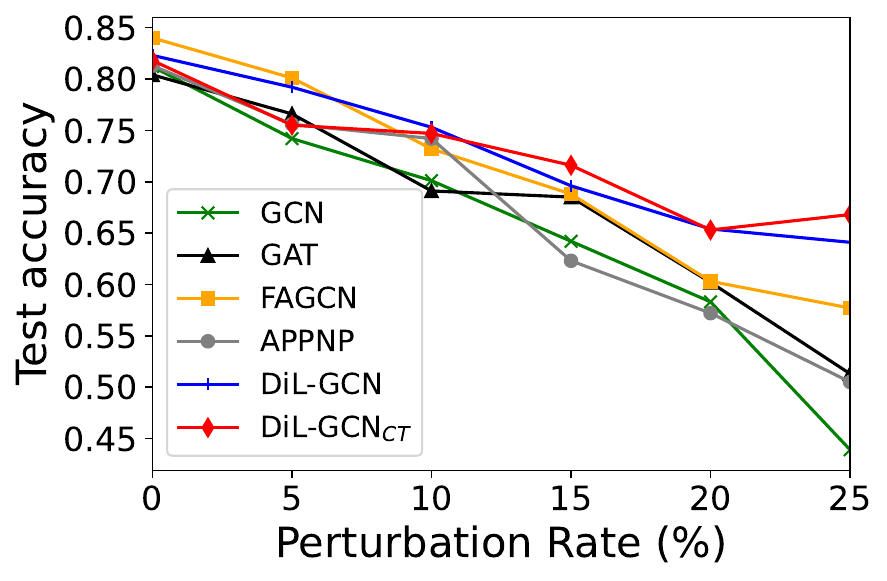} }\label{fig:att_cora}}
\subfloat[Chameleon]{{\includegraphics[width=0.31\linewidth]{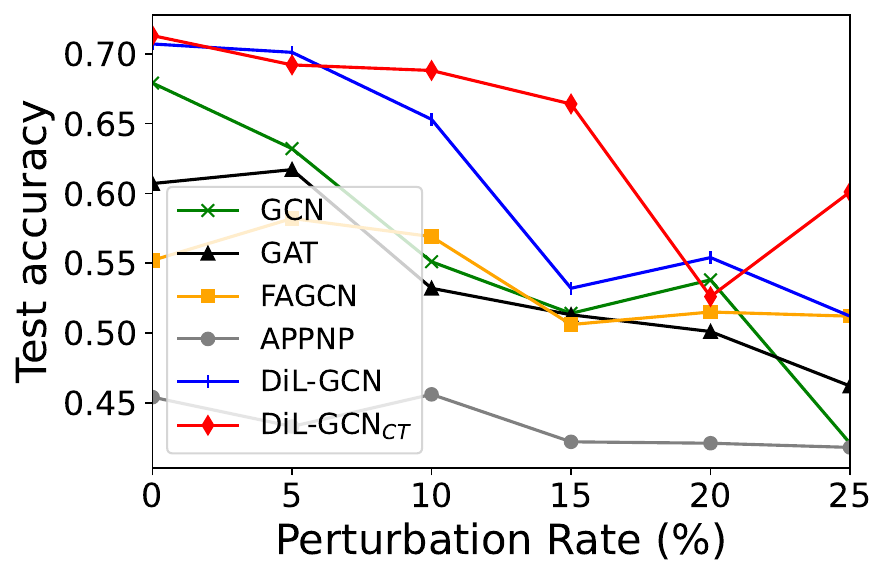} }\label{fig:att_chame}}
\caption{Node classification performance under structural attack.} 
\label{fig:att_defence}
\end{figure*}

\subsection{Node Classification}\label{sec:nc}
\Cref{tab:nc_result} reports the results of node classification on all ten graph datasets. We can see that our method DiL-GCN and DiL-GCN$_{CT}$ achieve new state-of-the-art results on 8 out of 10 datasets, and comparable results on two other datasets. Specifically, on datasets with strong heterophily ($h<0.5$), DiL-GCN$_{CT}$ achieves the best results on 5 out of 6 datasets and DiL-GCN achieves the best result on 1 dataset and second-best results on 4 datasets. In particular, on Texas, Squirrel, Chameleon and Cornell, we achieve 87.84\%, 54.22\%, 72.19\% and 87.30\% accuracies respectively, which are 2.98\%, 3.79\%, 4.23\% and 5.14\% relative improvements over previous state-of-the-art. Such strong performance on heterophilous graphs demonstrates the efficacy of our proposed framework based on DiLaplacian, and enhancing node proximity with commute time can further improve accuracy. On the dataset with intermediate heterophily ($h \approx 0.5$), our methods are competitive. Specifically, on deezer, H$_2$GCN and APPNP achieve the top two results, and DiL-GCN$_{CT}$ is the third. For homophilous graphs ($h > 0.7$), DiL-GCN achieves the best results on Citeseer and the second-best on CoraML and CoauthorCS, while DiL-GCN$_{CT}$ is sometimes inferior to DiL-GCN. We think the main reason is that for graphs with strong homophily, information from local neighbors is quite important because they contain sufficient intra-class information. Although the commute time provides a global perspective to find nodes with higher topological relevance, high-order neighbors may introduce some noise to the model in some homophilous graphs. Overall, most datasets highly benefit from every component we design. 

We make other observations as follows. Firstly, GCN+JK, GAT+JK and MixHop can not always outperform vanilla 2-layer GCN and GAT on heterophilous graphs, which suggests that coarsely aggregating distant neighbors is ineffective in filtering useless information for prediction. Secondly, MLP is a strong baseline on several heterophilous graphs (Texas, Wisconsin, and deezer), but it performs extremely poorly on other heterophilous graphs and all homophilous graphs, because it totally ignores the graph topology, which is important for both homophily and heterophily. Although AM-GCN tries to find a balance between topology and features by jointly training $k$NN graphs and original graphs, it can not consistently defeat MLP. We think the reason is AM-GCN can not fully utilize structural features, such as directionality and high-order neighbor information. Thirdly, APPNP, GPRGNN, and our models both are extensions of PageRank techniques. However, APPNP and GPRGNN which are respectively based on Personalized PageRank and Generalized PageRank highly depend on the original graph structure, such that they inevitably suffer from heterophily. Although they have a certain resistance to over-smoothing, they need deeper layers and all neighbors within per-hop are considered during message passing resulting in sub-optimal performance. 

\subsection{Robustness}
Besides naturally-occurring heterophily, heterophilous interactions may also be introduced as adversarial perturbations. Prior study~\cite{zhu2021relationship} on the robustness of GNNs reveals a relation between the vulnerabilities of GNNs to adversarial attacks and the increased presence of heterophily in perturbed graphs, i.e., imposing adversarial noise on the graph leads to a significant drop in the performance of GNNs is equivalent to modifying the original graph such that the degree of heterophily increases. 

Recall that our methods provide two strategies, which are considering the directionality of the graph and underlying long-distance correlations between nodes, to greatly improve the performance on node classification tasks when graph structure is not reliable. This motivates us to examine the potential benefit of our models on adversarial robustness. In this paper, we focus on perturbing the structure by adding or deleting edges, and evaluating the robustness of our methods on the node classification task. Specifically, we use metattack~\cite{zugner_adversarial_2019} to perform non-targeted attack, and follow the same experimental setting as~\cite{jin2021node}, i.e., the ratio of changed edges, from 0 to 25\% with a step of 5\%. We use GCN, GAT, APPNP and FAGCN as baselines and use the default hyperparameter settings in the authors' implementations. The hyperparameters of our methods are the same with \Cref{sec:nc}. We conduct the experiments on CoraML, Citeseer and Chameleon and report results in \Cref{fig:att_defence}. 

From \Cref{fig:att_citeseer} and \Cref{fig:att_cora}, we can observe that all methods have similar downward trends. On Citeseer, our DiL-GCN achieves the best defensive effect under 5\%$\sim$20\% perturbation rate. At 25\% perturbation rate, DiL-GCN$_{CT}$ surpasses DiL-GCN. On CoraML, FA-GCN achieves the best results under 5\% perturbation rate, our models show the best defensive performance when the rate  is above 10\%. FAGCN is competitive on this dataset from 5\% to 15\% perturbation rate. We can find that the class labels of CoraML are highly related to node feature according to the results of MLP in \Cref{tab:nc_result}, that is the reason that Feature-aware PageRank can help improve the robustness of our methods. From \Cref{fig:att_chame}, we can find that DiL-GCN$_{CT}$ has more than 10\% improvement over DiL-GCN and other baselines under 15\% perturbation rate, which demonstrates that the strategy of considering node proximity based on commute times can boost model robustness on the heterophilous graph.

\subsection{Component Analysis}  
\textit{\textbf{Directed v.s. Undirected.}} To verify the importance of considering the directionality, we study the values of the learnable parameters $\alpha$ and $\beta$ in DiL-GCN and DiL-GCN$_{CT}$. $\alpha$ and $\beta$ are used to build the combinatorial Laplacian to adaptively learn a `soft' directionality of the graph. If we manually set $\beta = 0$ and $\alpha = 1$, DiL-GCN and DiL-GCN$_{CT}$ will degrade into GraphSAGE applied on the undirected version of $G_c$ with mean aggregator, called GraphSAGE-${G_{c_u}}$. Next, we show the performance of GraphSAGE-${G_{c_u}}$, GraphSAGE and our models on Cornell, Chameleon and CoraML datasets in \Cref{tab:abl_dir}. Comparing the results of GraphSAGE-${G_{c_u}}$ and our DiL-GCN and DiL-GCN$_{CT}$, we can find that the convolutional operator either $\widehat{\mathcal{T}}$ or $\widehat{\mathcal{C}}_s$ plays an important role and can significantly influence the prediction accuracy. Besides, GraphSAGE-${G_{c_u}}$ consistently outperforms GraphSAGE, which proves that the structure of $G_c$ contains more useful local information, and considering long-distance correlations on the feature level is positive for both homophilous and heterophilous graphs.

\begin{table}[h]
\centering
\small
\caption{Experimental results for analyzing directionality.}
	\begin{adjustbox}{width=0.8\linewidth}

	\begin{tabular}{cccccccc}
	\toprule
	&\textbf{Cornell} & \textbf{Chameleon} & \textbf{CoraML}  \\ 
	\midrule
	GraphSAGE&	    80.09  &  58.73  & 80.03  \\    
	GraphSAGE-${G_{c_u}}$ & 82.67 & 63.38 & 81.57 \\
	\midrule
	DiL-GCN  &   85.41     & 70.70  & 82.30 \\ 
	DiL-GCN$_{CT}$ & 87.30  & 72.19  & 81.82 \\ 
	\bottomrule
	\end{tabular}
	\end{adjustbox}
\label{tab:abl_dir}
\end{table}

From \Cref{tab:abl_dir}, we can only know the `soft' treatment of directionality can improve the graph learning, however, how much does it matter is unclear. Hence, we further compare values of the learnable $\alpha$ and $\beta$ when achieving the best average accuracy on the validation set. It can reflect the bias of directionality and the influence of direction structure on the classification accuracy after training. From \Cref{fig:ablation_dir}, we observe that the learned $\alpha$ and $\beta$ vary across different datasets, and in most cases DiLaplacian-based convolutional operators $\widehat{\mathcal{T}}$ and $\widehat{\mathcal{C}}_s$ have greater contributions to the final representations than that based on the symmetrized adjacency matrix, while both of them have a non-negligible effect on the final results. It further proves that rather than absolutely treat the graph as directed or undirected, a soft adaptive directionality is more appropriate.

\begin{figure}[h]
	\centering
    \subfloat[DiL-GCN]{{\includegraphics[width=0.495\linewidth]{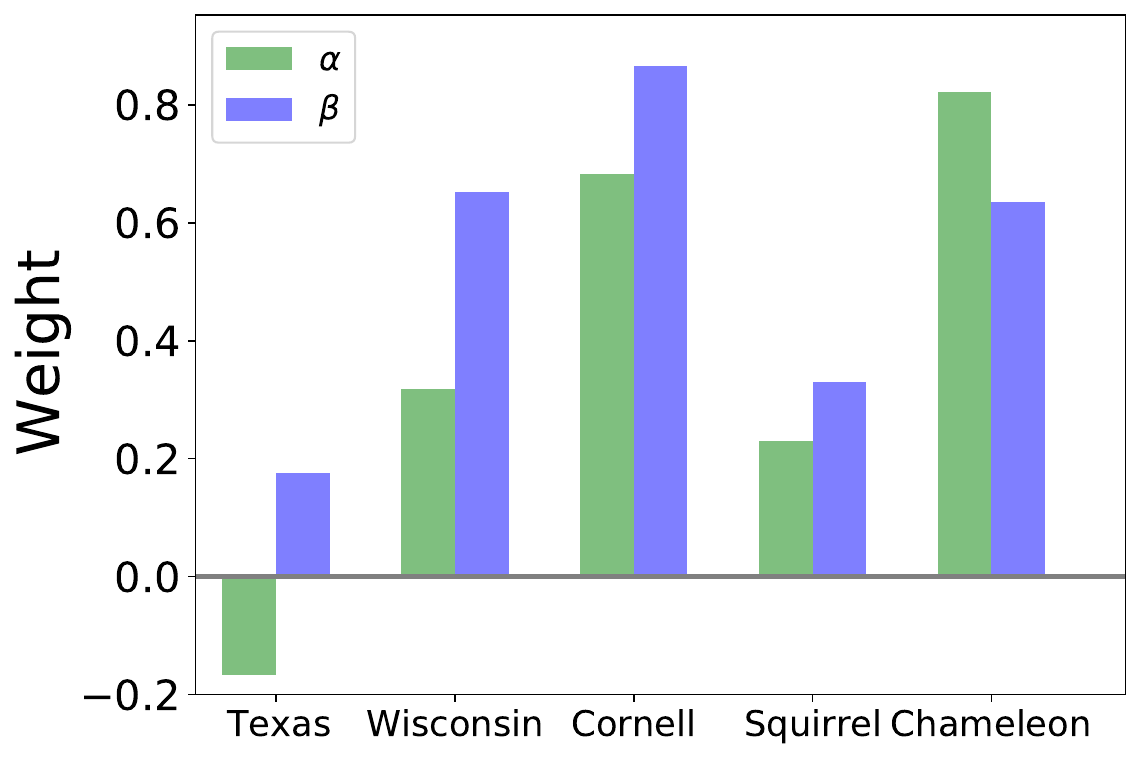}\label{fig:abl_dir_diglacian}}}
    \subfloat[DiL-GCN$_{CT}$]{{\includegraphics[width=0.475\linewidth]{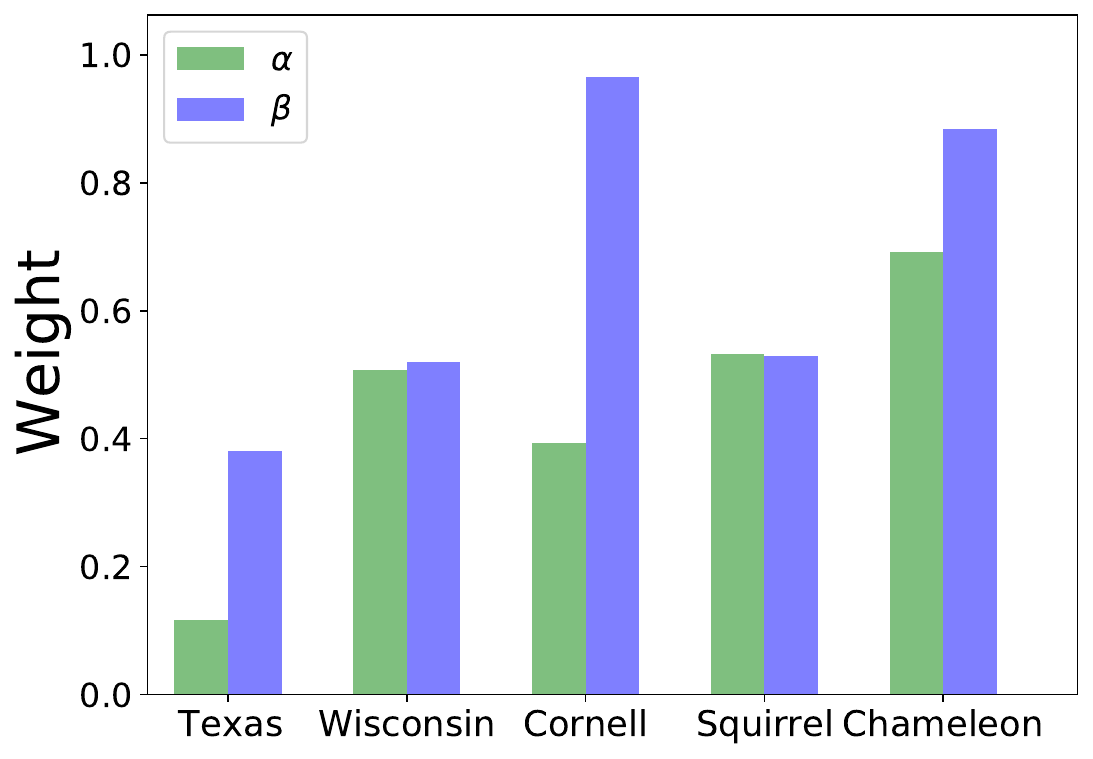}
    }\label{fig:para_wis}}
    \caption{Values of $\alpha$ and $\beta$ when the validation set achieves the best accuracy.}
\label{fig:ablation_dir}
\end{figure}

\textit{\textbf{Feature-aware PageRank v.s. PageRank.}} In our models, we propose to construct a combined graph $\tilde{G}_c$ based on Feature-aware PageRank to guarantee the irreducibility and aperiodicity, so that the DiLaplacian can be defined. On the other hand, the classic PageRank transition matrix $\mathbf{P}_{pr}=\gamma \mathbf{P} + (1-\gamma) \frac{e e^\top}{N}$ is also I\&A and it yields a dense and feature-independent convolutional operator. To further verify the importance of preserving proximity in the feature space, we replace $\tilde{\mathbf{P}}_{fpr}$ and $\Pi$ in \cref{eq:aug_DiLaplacian} to $\mathbf{P}_{pr}$ and its stationary distribution, and $\tilde{\mathbf{D}}_{c}$ is replaced with the out-degree matrix of the original graph. We use `(w/o feat.)' to represent this variant. This variant can help us verify the effectiveness of FPR. Further, we apply FPR to APPNP (APPNP$_{fpr}$) to test the transferrability of FPR. Results are shown in \Cref{tab:abl_page}, from which we can conclude that the feature-independent variants is sub-optimal, and FPR can be seen as a plug-and-play module so that can be transferred to other PageRank-based models flexibly. 

\begin{table}[h!]
\centering
\small
\caption{Impact of feature-aware PageRank.}
	\begin{adjustbox}{width=0.85\linewidth}
\begin{tabular}{cccccccc}
	\toprule
	&\textbf{Cornell} & \textbf{Chameleon} & \textbf{CoraML}\\
	\midrule
    APPNP &      70.96 &  45.37 &  81.31                  \\
    APPNP$_{fpr}$ & 75.83 & 48.64 & 82.37      \\
	DiL-GCN(w/o fpr.) & 84.03  & 59.65 & 79.59       \\
	DiL-GCN$_{CT}$(w/o fpr.)& 84.59  & 60.84     & 78.57           \\
	DiL-GCN		                & 85.41   & 70.70     & 82.30           \\
	DiL-GCN$_{CT}$     	            & 87.30  & 71.33     & 81.82        \\
	\bottomrule
\end{tabular}
	\end{adjustbox}
\label{tab:abl_page}
\end{table}



\textit{\textbf{Parameter Sensitivity.}} In \Cref{sec:fpr}, we introduce the approximate similarity sorting to generate an irreducible graph $G_s$ and then combine $G_s$ with the original graph $G$ to construct an irreducible combined graph $G_c$, where $window\_size$ is set to a small value to guarantee sparsity. Hence, we explore the effect of $window\_size$ on the performance of node classification task. We choose $window\_size$ from 1 to 5 and conduct experiments with different $window\_size$ on 4 datasets, and report results in \Cref{fig:para_sen}. We observe that with the increase of $window\_size$, the performance increases until reaches at a peak and then decreases. This is reasonable as suitable number of distant nodes are capable of incorporating proper global information which makes nodes more discriminative, while excessive distant nodes may introduce noise which slow down the training speed due to density and hurt the generalization ability. Thus, $window\_size$ should be carefully decided to achieve the optimal performance on the validation set.

\begin{figure}[htbp]
\begin{minipage}[t]{1\linewidth}
\centering
\subfloat[Cornell]{{\includegraphics[width=0.25\linewidth]{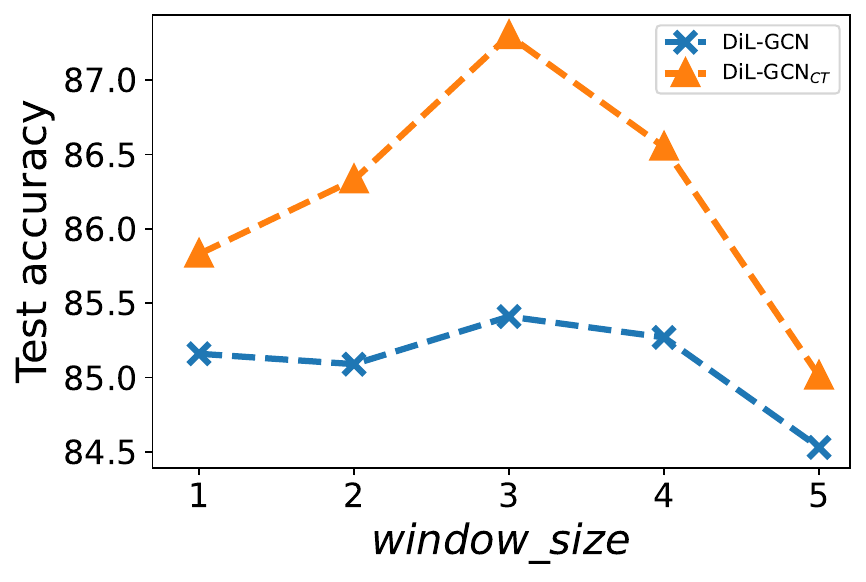} }\label{fig:para_texas}}
\subfloat[Wisconsin]{{\includegraphics[width=0.25\linewidth]{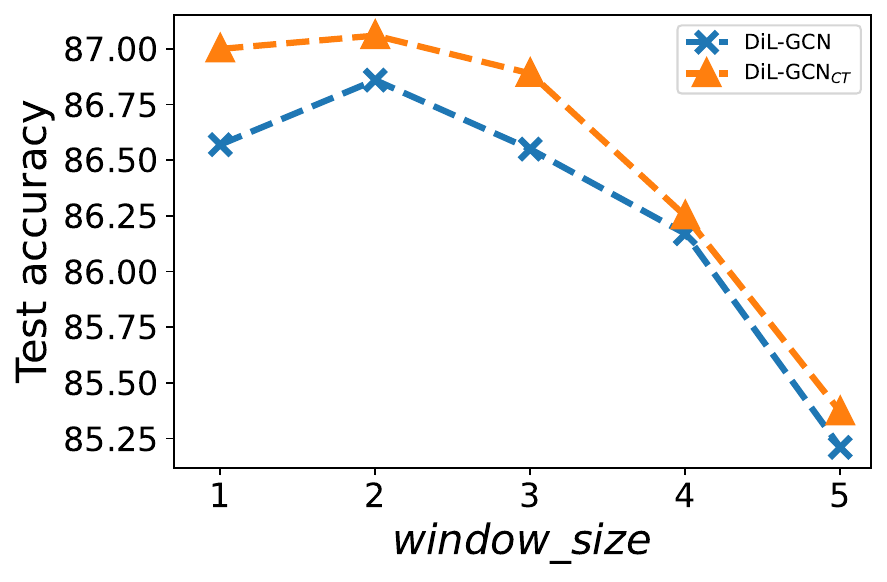}}\label{fig:para_wis}}
\subfloat[CoraML]{{\includegraphics[width=0.25\linewidth]{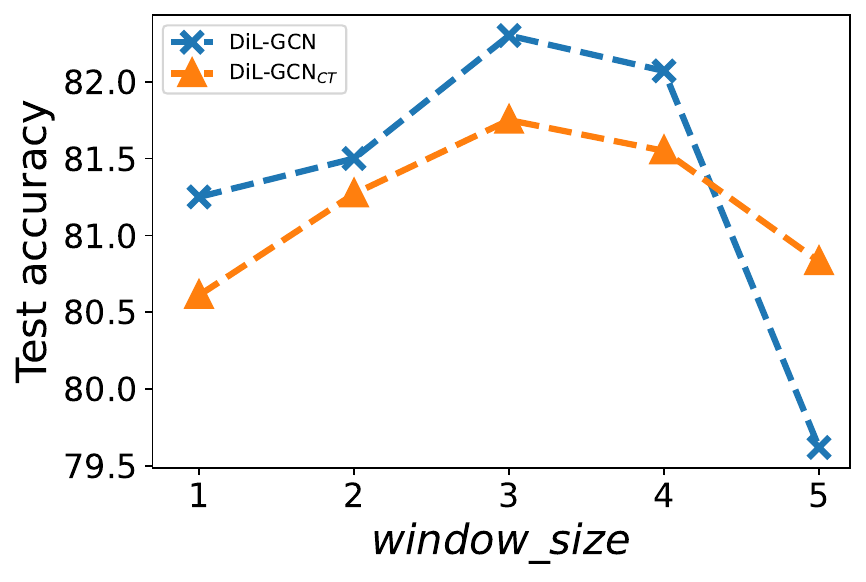} }\label{fig:para_cora}}
\subfloat[Squirrel]{{\includegraphics[width=0.25\linewidth]{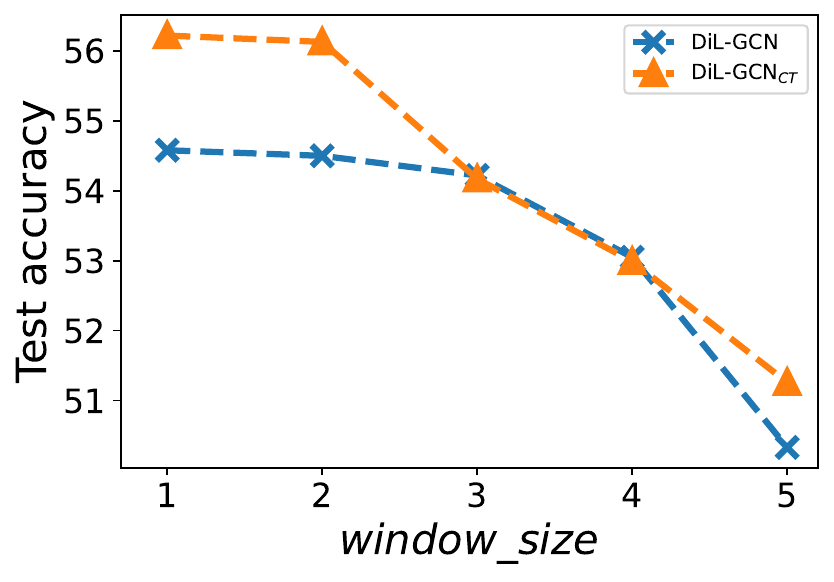} }\label{fig:para_squ}}
\caption{Classification results under different $window\_size$.} 
\label{fig:para_sen}
\end{minipage}
\end{figure}

\section{Conclusion}
Our work aims to improve the performance of GNN on heterophilous graphs by considering directionality and long-distance correlations. We generalize the graph Laplacian to digraph by defining DiLaplacian based on the proposed Feature-aware PageRank, such that the direction structure and feature-level long-distance correlations can be preserved simultaneously. The graph convolutional operator is based on the DiLaplacian and the symmetrized adjacency matrix such that the model can adaptively learn the directionality of the graph. Further, we define a measure of node proximity based on commute times and combine it into our model, which can capture the long-distance correlations between nodes on the topology level. The extensive experiments confirm the effectiveness and robustness of our models.

\balance
\bibliographystyle{ACM-Reference-Format}
\bibliography{reference}


\begin{thebibliography}{53}


\ifx \showCODEN    \undefined \def \showCODEN     #1{\unskip}     \fi
\ifx \showDOI      \undefined \def \showDOI       #1{#1}\fi
\ifx \showISBNx    \undefined \def \showISBNx     #1{\unskip}     \fi
\ifx \showISBNxiii \undefined \def \showISBNxiii  #1{\unskip}     \fi
\ifx \showISSN     \undefined \def \showISSN      #1{\unskip}     \fi
\ifx \showLCCN     \undefined \def \showLCCN      #1{\unskip}     \fi
\ifx \shownote     \undefined \def \shownote      #1{#1}          \fi
\ifx \showarticletitle \undefined \def \showarticletitle #1{#1}   \fi
\ifx \showURL      \undefined \def \showURL       {\relax}        \fi
\providecommand\bibfield[2]{#2}
\providecommand\bibinfo[2]{#2}
\providecommand\natexlab[1]{#1}
\providecommand\showeprint[2][]{arXiv:#2}

\bibitem[Abu-El-Haija et~al\mbox{.}(2019)]%
        {abu2019mixhop}
\bibfield{author}{\bibinfo{person}{Sami Abu-El-Haija}, \bibinfo{person}{Bryan
  Perozzi}, \bibinfo{person}{Amol Kapoor}, \bibinfo{person}{Nazanin
  Alipourfard}, \bibinfo{person}{Kristina Lerman}, \bibinfo{person}{Hrayr
  Harutyunyan}, \bibinfo{person}{Greg Ver~Steeg}, {and} \bibinfo{person}{Aram
  Galstyan}.} \bibinfo{year}{2019}\natexlab{}.
\newblock \showarticletitle{Mixhop: Higher-order graph convolutional
  architectures via sparsified neighborhood mixing}. In
  \bibinfo{booktitle}{\emph{international conference on machine learning}}.
  PMLR, \bibinfo{pages}{21--29}.
\newblock


\bibitem[Aldous and Fill(2002)]%
        {aldous2002reversible}
\bibfield{author}{\bibinfo{person}{David Aldous} {and} \bibinfo{person}{Jim
  Fill}.} \bibinfo{year}{2002}\natexlab{}.
\newblock \bibinfo{title}{Reversible Markov chains and random walks on graphs}.
\newblock
\newblock


\bibitem[Bo et~al\mbox{.}(2021)]%
        {fagcn2021}
\bibfield{author}{\bibinfo{person}{Deyu Bo}, \bibinfo{person}{Xiao Wang},
  \bibinfo{person}{Chuan Shi}, {and} \bibinfo{person}{Huawei Shen}.}
  \bibinfo{year}{2021}\natexlab{}.
\newblock \showarticletitle{Beyond Low-frequency Information in Graph
  Convolutional Networks}. In \bibinfo{booktitle}{\emph{{AAAI}}}.
  \bibinfo{publisher}{{AAAI} Press}.
\newblock


\bibitem[Bojchevski and Günnemann(2018)]%
        {bojchevski2018deep}
\bibfield{author}{\bibinfo{person}{Aleksandar Bojchevski} {and}
  \bibinfo{person}{Stephan Günnemann}.} \bibinfo{year}{2018}\natexlab{}.
\newblock \showarticletitle{Deep Gaussian Embedding of Graphs: Unsupervised
  Inductive Learning via Ranking}. In \bibinfo{booktitle}{\emph{International
  Conference on Learning Representations}}.
\newblock
\urldef\tempurl%
\url{https://openreview.net/forum?id=r1ZdKJ-0W}
\showURL{%
\tempurl}


\bibitem[Chen et~al\mbox{.}(2020)]%
        {chen2020iterative}
\bibfield{author}{\bibinfo{person}{Yu Chen}, \bibinfo{person}{Lingfei Wu},
  {and} \bibinfo{person}{Mohammed Zaki}.} \bibinfo{year}{2020}\natexlab{}.
\newblock \showarticletitle{Iterative Deep Graph Learning for Graph Neural
  Networks: Better and Robust Node Embeddings}.
\newblock \bibinfo{journal}{\emph{Advances in Neural Information Processing
  Systems}}  \bibinfo{volume}{33} (\bibinfo{year}{2020}).
\newblock


\bibitem[Chien et~al\mbox{.}(2021)]%
        {chien2021adaptive}
\bibfield{author}{\bibinfo{person}{Eli Chien}, \bibinfo{person}{Jianhao Peng},
  \bibinfo{person}{Pan Li}, {and} \bibinfo{person}{Olgica Milenkovic}.}
  \bibinfo{year}{2021}\natexlab{}.
\newblock \showarticletitle{Adaptive Universal Generalized PageRank Graph
  Neural Network}. In \bibinfo{booktitle}{\emph{International Conference on
  Learning Representations}}.
\newblock
\urldef\tempurl%
\url{https://openreview.net/forum?id=n6jl7fLxrP}
\showURL{%
\tempurl}


\bibitem[Chung(2005)]%
        {chung2005laplacians}
\bibfield{author}{\bibinfo{person}{Fan Chung}.}
  \bibinfo{year}{2005}\natexlab{}.
\newblock \showarticletitle{Laplacians and the Cheeger inequality for directed
  graphs}.
\newblock \bibinfo{journal}{\emph{Annals of Combinatorics}}
  \bibinfo{volume}{9}, \bibinfo{number}{1} (\bibinfo{year}{2005}),
  \bibinfo{pages}{1--19}.
\newblock


\bibitem[Defferrard et~al\mbox{.}(2016)]%
        {defferrard2016convolutional}
\bibfield{author}{\bibinfo{person}{Micha{\"e}l Defferrard},
  \bibinfo{person}{Xavier Bresson}, {and} \bibinfo{person}{Pierre
  Vandergheynst}.} \bibinfo{year}{2016}\natexlab{}.
\newblock \showarticletitle{Convolutional neural networks on graphs with fast
  localized spectral filtering}.
\newblock \bibinfo{journal}{\emph{Advances in neural information processing
  systems}}  \bibinfo{volume}{29} (\bibinfo{year}{2016}),
  \bibinfo{pages}{3844--3852}.
\newblock


\bibitem[Dong et~al\mbox{.}(2021)]%
        {dong2021graph}
\bibfield{author}{\bibinfo{person}{Yushun Dong}, \bibinfo{person}{Kaize Ding},
  \bibinfo{person}{Brian Jalaian}, \bibinfo{person}{Shuiwang Ji}, {and}
  \bibinfo{person}{Jundong Li}.} \bibinfo{year}{2021}\natexlab{}.
\newblock \showarticletitle{Graph Neural Networks with Adaptive Frequency
  Response Filter}.
\newblock \bibinfo{journal}{\emph{arXiv preprint arXiv:2104.12840}}
  (\bibinfo{year}{2021}).
\newblock


\bibitem[Franceschi et~al\mbox{.}(2019)]%
        {franceschi2019learning}
\bibfield{author}{\bibinfo{person}{Luca Franceschi}, \bibinfo{person}{Mathias
  Niepert}, \bibinfo{person}{Massimiliano Pontil}, {and} \bibinfo{person}{Xiao
  He}.} \bibinfo{year}{2019}\natexlab{}.
\newblock \showarticletitle{Learning Discrete Structures for Graph Neural
  Networks}. In \bibinfo{booktitle}{\emph{Proceedings of the 36th International
  Conference on Machine Learning}}.
\newblock


\bibitem[Frasca et~al\mbox{.}(2020)]%
        {frasca2020sign}
\bibfield{author}{\bibinfo{person}{Fabrizio Frasca}, \bibinfo{person}{Emanuele
  Rossi}, \bibinfo{person}{Davide Eynard}, \bibinfo{person}{Ben Chamberlain},
  \bibinfo{person}{Michael Bronstein}, {and} \bibinfo{person}{Federico Monti}.}
  \bibinfo{year}{2020}\natexlab{}.
\newblock \showarticletitle{Sign: Scalable inception graph neural networks}.
\newblock \bibinfo{journal}{\emph{arXiv preprint arXiv:2004.11198}}
  (\bibinfo{year}{2020}).
\newblock


\bibitem[Gilmer et~al\mbox{.}(2017)]%
        {gilmer2017neural}
\bibfield{author}{\bibinfo{person}{Justin Gilmer}, \bibinfo{person}{Samuel~S
  Schoenholz}, \bibinfo{person}{Patrick~F Riley}, \bibinfo{person}{Oriol
  Vinyals}, {and} \bibinfo{person}{George~E Dahl}.}
  \bibinfo{year}{2017}\natexlab{}.
\newblock \showarticletitle{Neural message passing for quantum chemistry}. In
  \bibinfo{booktitle}{\emph{International conference on machine learning}}.
  PMLR, \bibinfo{pages}{1263--1272}.
\newblock


\bibitem[Gleich(2006)]%
        {gleich2006hierarchical}
\bibfield{author}{\bibinfo{person}{David Gleich}.}
  \bibinfo{year}{2006}\natexlab{}.
\newblock \showarticletitle{Hierarchical directed spectral graph partitioning}.
\newblock \bibinfo{journal}{\emph{Information Networks}}  \bibinfo{volume}{443}
  (\bibinfo{year}{2006}).
\newblock


\bibitem[Glorot and Bengio(2010)]%
        {glorot2010understanding}
\bibfield{author}{\bibinfo{person}{Xavier Glorot} {and} \bibinfo{person}{Yoshua
  Bengio}.} \bibinfo{year}{2010}\natexlab{}.
\newblock \showarticletitle{Understanding the difficulty of training deep
  feedforward neural networks}. In \bibinfo{booktitle}{\emph{Proceedings of the
  thirteenth international conference on artificial intelligence and
  statistics}}. JMLR Workshop and Conference Proceedings,
  \bibinfo{pages}{249--256}.
\newblock


\bibitem[Hamilton et~al\mbox{.}(2017)]%
        {hamilton2017inductive}
\bibfield{author}{\bibinfo{person}{William~L Hamilton}, \bibinfo{person}{Rex
  Ying}, {and} \bibinfo{person}{Jure Leskovec}.}
  \bibinfo{year}{2017}\natexlab{}.
\newblock \showarticletitle{Inductive representation learning on large graphs}.
  In \bibinfo{booktitle}{\emph{Proceedings of the 31st International Conference
  on Neural Information Processing Systems}}. \bibinfo{pages}{1025--1035}.
\newblock


\bibitem[Jiang et~al\mbox{.}(2019)]%
        {jiang2019semi}
\bibfield{author}{\bibinfo{person}{Bo Jiang}, \bibinfo{person}{Ziyan Zhang},
  \bibinfo{person}{Doudou Lin}, \bibinfo{person}{Jin Tang}, {and}
  \bibinfo{person}{Bin Luo}.} \bibinfo{year}{2019}\natexlab{}.
\newblock \showarticletitle{Semi-supervised learning with graph
  learning-convolutional networks}. In \bibinfo{booktitle}{\emph{Proceedings of
  the IEEE/CVF Conference on Computer Vision and Pattern Recognition}}.
  \bibinfo{pages}{11313--11320}.
\newblock


\bibitem[Jin et~al\mbox{.}(2021)]%
        {jin2021node}
\bibfield{author}{\bibinfo{person}{Wei Jin}, \bibinfo{person}{Tyler Derr},
  \bibinfo{person}{Yiqi Wang}, \bibinfo{person}{Yao Ma}, \bibinfo{person}{Zitao
  Liu}, {and} \bibinfo{person}{Jiliang Tang}.} \bibinfo{year}{2021}\natexlab{}.
\newblock \showarticletitle{Node similarity preserving graph convolutional
  networks}. In \bibinfo{booktitle}{\emph{Proceedings of the 14th ACM
  International Conference on Web Search and Data Mining}}.
  \bibinfo{pages}{148--156}.
\newblock


\bibitem[Kingma and Ba(2014)]%
        {kingma2014adam}
\bibfield{author}{\bibinfo{person}{Diederik~P Kingma} {and}
  \bibinfo{person}{Jimmy Ba}.} \bibinfo{year}{2014}\natexlab{}.
\newblock \showarticletitle{Adam: A method for stochastic optimization}.
\newblock \bibinfo{journal}{\emph{arXiv preprint arXiv:1412.6980}}
  (\bibinfo{year}{2014}).
\newblock


\bibitem[Kipf and Welling(2017)]%
        {kipf2017semi}
\bibfield{author}{\bibinfo{person}{Thomas~N. Kipf} {and} \bibinfo{person}{Max
  Welling}.} \bibinfo{year}{2017}\natexlab{}.
\newblock \showarticletitle{Semi-Supervised Classification with Graph
  Convolutional Networks}. In \bibinfo{booktitle}{\emph{International
  Conference on Learning Representations (ICLR)}}.
\newblock


\bibitem[Klicpera et~al\mbox{.}(2019a)]%
        {klicpera_predict_2019}
\bibfield{author}{\bibinfo{person}{Johannes Klicpera},
  \bibinfo{person}{Aleksandar Bojchevski}, {and} \bibinfo{person}{Stephan
  G{\"u}nnemann}.} \bibinfo{year}{2019}\natexlab{a}.
\newblock \showarticletitle{Predict then Propagate: Graph Neural Networks meet
  Personalized PageRank}. In \bibinfo{booktitle}{\emph{International Conference
  on Learning Representations (ICLR)}}.
\newblock


\bibitem[Klicpera et~al\mbox{.}(2019b)]%
        {klicpera_diffusion_2019}
\bibfield{author}{\bibinfo{person}{Johannes Klicpera}, \bibinfo{person}{Stefan
  Wei{\ss}enberger}, {and} \bibinfo{person}{Stephan G{\"u}nnemann}.}
  \bibinfo{year}{2019}\natexlab{b}.
\newblock \showarticletitle{Diffusion Improves Graph Learning}. In
  \bibinfo{booktitle}{\emph{Conference on Neural Information Processing Systems
  (NeurIPS)}}.
\newblock


\bibitem[Lehoucq et~al\mbox{.}(1998)]%
        {lehoucq1998arpack}
\bibfield{author}{\bibinfo{person}{Richard~B Lehoucq}, \bibinfo{person}{Danny~C
  Sorensen}, {and} \bibinfo{person}{Chao Yang}.}
  \bibinfo{year}{1998}\natexlab{}.
\newblock \bibinfo{booktitle}{\emph{ARPACK users' guide: solution of
  large-scale eigenvalue problems with implicitly restarted Arnoldi methods}}.
\newblock \bibinfo{publisher}{SIAM}.
\newblock


\bibitem[{Li} et~al\mbox{.}(2018)]%
        {li2018deeper}
\bibfield{author}{\bibinfo{person}{Q. {Li}}, \bibinfo{person}{Z. {Han}}, {and}
  \bibinfo{person}{X.-M. {Wu}}.} \bibinfo{year}{2018}\natexlab{}.
\newblock \showarticletitle{{Deeper Insights into Graph Convolutional Networks
  for Semi-Supervised Learning}}. In \bibinfo{booktitle}{\emph{The
  Thirty-Second AAAI Conference on Artificial Intelligence}}. AAAI.
\newblock


\bibitem[Li and Zhang(2012)]%
        {li2012digraph}
\bibfield{author}{\bibinfo{person}{Yanhua Li} {and} \bibinfo{person}{Zhi-Li
  Zhang}.} \bibinfo{year}{2012}\natexlab{}.
\newblock \showarticletitle{Digraph laplacian and the degree of asymmetry}.
\newblock \bibinfo{journal}{\emph{Internet Mathematics}} \bibinfo{volume}{8},
  \bibinfo{number}{4} (\bibinfo{year}{2012}), \bibinfo{pages}{381--401}.
\newblock


\bibitem[Lim et~al\mbox{.}(2021)]%
        {lim2021new}
\bibfield{author}{\bibinfo{person}{Derek Lim}, \bibinfo{person}{Xiuyu Li},
  \bibinfo{person}{Felix Hohne}, {and} \bibinfo{person}{Ser-Nam Lim}.}
  \bibinfo{year}{2021}\natexlab{}.
\newblock \showarticletitle{New Benchmarks for Learning on Non-Homophilous
  Graphs}.
\newblock \bibinfo{journal}{\emph{arXiv preprint arXiv:2104.01404}}
  (\bibinfo{year}{2021}).
\newblock


\bibitem[Ming~Chen et~al\mbox{.}(2020)]%
        {chenWHDL2020gcnii}
\bibfield{author}{\bibinfo{person}{Zhewei~Wei Ming~Chen},
  \bibinfo{person}{Bolin~Ding Zengfeng~Huang}, {and} \bibinfo{person}{Yaliang
  Li}.} \bibinfo{year}{2020}\natexlab{}.
\newblock \showarticletitle{Simple and Deep Graph Convolutional Networks}.
\newblock  (\bibinfo{year}{2020}).
\newblock


\bibitem[Page et~al\mbox{.}(1999)]%
        {page1999pagerank}
\bibfield{author}{\bibinfo{person}{Lawrence Page}, \bibinfo{person}{Sergey
  Brin}, \bibinfo{person}{Rajeev Motwani}, {and} \bibinfo{person}{Terry
  Winograd}.} \bibinfo{year}{1999}\natexlab{}.
\newblock \bibinfo{booktitle}{\emph{The PageRank citation ranking: Bringing
  order to the web.}}
\newblock \bibinfo{type}{{T}echnical {R}eport}. \bibinfo{institution}{Stanford
  InfoLab}.
\newblock


\bibitem[Pei et~al\mbox{.}(2019)]%
        {pei2019geom}
\bibfield{author}{\bibinfo{person}{Hongbin Pei}, \bibinfo{person}{Bingzhe Wei},
  \bibinfo{person}{Kevin Chen-Chuan Chang}, \bibinfo{person}{Yu Lei}, {and}
  \bibinfo{person}{Bo Yang}.} \bibinfo{year}{2019}\natexlab{}.
\newblock \showarticletitle{Geom-GCN: Geometric Graph Convolutional Networks}.
  In \bibinfo{booktitle}{\emph{International Conference on Learning
  Representations}}.
\newblock


\bibitem[Poole(2014)]%
        {poole2014linear}
\bibfield{author}{\bibinfo{person}{David Poole}.}
  \bibinfo{year}{2014}\natexlab{}.
\newblock \bibinfo{booktitle}{\emph{Linear algebra: A modern introduction}}.
\newblock \bibinfo{publisher}{Cengage Learning}.
\newblock


\bibitem[Qu et~al\mbox{.}(2019)]%
        {qu2019gmnn}
\bibfield{author}{\bibinfo{person}{Meng Qu}, \bibinfo{person}{Yoshua Bengio},
  {and} \bibinfo{person}{Jian Tang}.} \bibinfo{year}{2019}\natexlab{}.
\newblock \showarticletitle{Gmnn: Graph markov neural networks}. In
  \bibinfo{booktitle}{\emph{International conference on machine learning}}.
  PMLR, \bibinfo{pages}{5241--5250}.
\newblock


\bibitem[Shchur et~al\mbox{.}(2018)]%
        {shchur2018pitfalls}
\bibfield{author}{\bibinfo{person}{Oleksandr Shchur},
  \bibinfo{person}{Maximilian Mumme}, \bibinfo{person}{Aleksandar Bojchevski},
  {and} \bibinfo{person}{Stephan G{\"u}nnemann}.}
  \bibinfo{year}{2018}\natexlab{}.
\newblock \showarticletitle{Pitfalls of graph neural network evaluation}.
\newblock \bibinfo{journal}{\emph{arXiv preprint arXiv:1811.05868}}
  (\bibinfo{year}{2018}).
\newblock


\bibitem[Singh et~al\mbox{.}(2016)]%
        {singh2016graph}
\bibfield{author}{\bibinfo{person}{Rahul Singh}, \bibinfo{person}{Abhishek
  Chakraborty}, {and} \bibinfo{person}{BS Manoj}.}
  \bibinfo{year}{2016}\natexlab{}.
\newblock \showarticletitle{Graph Fourier transform based on directed
  Laplacian}. In \bibinfo{booktitle}{\emph{2016 International Conference on
  Signal Processing and Communications (SPCOM)}}. IEEE, \bibinfo{pages}{1--5}.
\newblock


\bibitem[Srivastava et~al\mbox{.}(2014)]%
        {srivastava2014dropout}
\bibfield{author}{\bibinfo{person}{Nitish Srivastava},
  \bibinfo{person}{Geoffrey Hinton}, \bibinfo{person}{Alex Krizhevsky},
  \bibinfo{person}{Ilya Sutskever}, {and} \bibinfo{person}{Ruslan
  Salakhutdinov}.} \bibinfo{year}{2014}\natexlab{}.
\newblock \showarticletitle{Dropout: a simple way to prevent neural networks
  from overfitting}.
\newblock \bibinfo{journal}{\emph{The journal of machine learning research}}
  \bibinfo{volume}{15}, \bibinfo{number}{1} (\bibinfo{year}{2014}),
  \bibinfo{pages}{1929--1958}.
\newblock


\bibitem[Tong et~al\mbox{.}(2020a)]%
        {tong2020digraph}
\bibfield{author}{\bibinfo{person}{Zekun Tong}, \bibinfo{person}{Yuxuan Liang},
  \bibinfo{person}{Changsheng Sun}, \bibinfo{person}{Xinke Li},
  \bibinfo{person}{David Rosenblum}, {and} \bibinfo{person}{Andrew Lim}.}
  \bibinfo{year}{2020}\natexlab{a}.
\newblock \showarticletitle{Digraph inception convolutional networks}.
\newblock \bibinfo{journal}{\emph{Advances in neural information processing
  systems}}  \bibinfo{volume}{33} (\bibinfo{year}{2020}).
\newblock


\bibitem[Tong et~al\mbox{.}(2020b)]%
        {tong2020directed}
\bibfield{author}{\bibinfo{person}{Zekun Tong}, \bibinfo{person}{Yuxuan Liang},
  \bibinfo{person}{Changsheng Sun}, \bibinfo{person}{David~S Rosenblum}, {and}
  \bibinfo{person}{Andrew Lim}.} \bibinfo{year}{2020}\natexlab{b}.
\newblock \showarticletitle{Directed graph convolutional network}.
\newblock \bibinfo{journal}{\emph{arXiv preprint arXiv:2004.13970}}
  (\bibinfo{year}{2020}).
\newblock


\bibitem[Veli{\v{c}}kovi{\'{c}} et~al\mbox{.}(2018)]%
        {velickovic2018graph}
\bibfield{author}{\bibinfo{person}{Petar Veli{\v{c}}kovi{\'{c}}},
  \bibinfo{person}{Guillem Cucurull}, \bibinfo{person}{Arantxa Casanova},
  \bibinfo{person}{Adriana Romero}, \bibinfo{person}{Pietro Li{\`{o}}}, {and}
  \bibinfo{person}{Yoshua Bengio}.} \bibinfo{year}{2018}\natexlab{}.
\newblock \showarticletitle{{Graph Attention Networks}}.
\newblock \bibinfo{journal}{\emph{International Conference on Learning
  Representations}} (\bibinfo{year}{2018}).
\newblock
\urldef\tempurl%
\url{https://openreview.net/forum?id=rJXMpikCZ}
\showURL{%
\tempurl}
\newblock
\shownote{accepted as poster}.


\bibitem[Wang et~al\mbox{.}(2021a)]%
        {wang2021graph}
\bibfield{author}{\bibinfo{person}{Ruijia Wang}, \bibinfo{person}{Shuai Mou},
  \bibinfo{person}{Xiao Wang}, \bibinfo{person}{Wanpeng Xiao},
  \bibinfo{person}{Qi Ju}, \bibinfo{person}{Chuan Shi}, {and}
  \bibinfo{person}{Xing Xie}.} \bibinfo{year}{2021}\natexlab{a}.
\newblock \showarticletitle{Graph Structure Estimation Neural Networks}. In
  \bibinfo{booktitle}{\emph{Proceedings of the Web Conference 2021}}.
  \bibinfo{pages}{342--353}.
\newblock


\bibitem[Wang et~al\mbox{.}(2021b)]%
        {wang2021powerful}
\bibfield{author}{\bibinfo{person}{Tao Wang}, \bibinfo{person}{Rui Wang},
  \bibinfo{person}{Di Jin}, \bibinfo{person}{Dongxiao He}, {and}
  \bibinfo{person}{Yuxiao Huang}.} \bibinfo{year}{2021}\natexlab{b}.
\newblock \showarticletitle{Powerful Graph Convolutioal Networks with Adaptive
  Propagation Mechanism for Homophily and Heterophily}.
\newblock \bibinfo{journal}{\emph{arXiv preprint arXiv:2112.13562}}
  (\bibinfo{year}{2021}).
\newblock


\bibitem[Wang et~al\mbox{.}(2020)]%
        {wang2020gcn}
\bibfield{author}{\bibinfo{person}{Xiao Wang}, \bibinfo{person}{Meiqi Zhu},
  \bibinfo{person}{Deyu Bo}, \bibinfo{person}{Peng Cui}, \bibinfo{person}{Chuan
  Shi}, {and} \bibinfo{person}{Jian Pei}.} \bibinfo{year}{2020}\natexlab{}.
\newblock \showarticletitle{Am-gcn: Adaptive multi-channel graph convolutional
  networks}. In \bibinfo{booktitle}{\emph{Proceedings of the 26th ACM SIGKDD
  International conference on knowledge discovery \& data mining}}.
  \bibinfo{pages}{1243--1253}.
\newblock


\bibitem[Wu et~al\mbox{.}(2019)]%
        {wu2019simplifying}
\bibfield{author}{\bibinfo{person}{Felix Wu}, \bibinfo{person}{Amauri Souza},
  \bibinfo{person}{Tianyi Zhang}, \bibinfo{person}{Christopher Fifty},
  \bibinfo{person}{Tao Yu}, {and} \bibinfo{person}{Kilian Weinberger}.}
  \bibinfo{year}{2019}\natexlab{}.
\newblock \showarticletitle{Simplifying graph convolutional networks}. In
  \bibinfo{booktitle}{\emph{International conference on machine learning}}.
  PMLR, \bibinfo{pages}{6861--6871}.
\newblock


\bibitem[Xu et~al\mbox{.}(2018)]%
        {xu2018representation}
\bibfield{author}{\bibinfo{person}{Keyulu Xu}, \bibinfo{person}{Chengtao Li},
  \bibinfo{person}{Yonglong Tian}, \bibinfo{person}{Tomohiro Sonobe},
  \bibinfo{person}{Ken-ichi Kawarabayashi}, {and} \bibinfo{person}{Stefanie
  Jegelka}.} \bibinfo{year}{2018}\natexlab{}.
\newblock \showarticletitle{Representation learning on graphs with jumping
  knowledge networks}. In \bibinfo{booktitle}{\emph{International Conference on
  Machine Learning}}. PMLR, \bibinfo{pages}{5453--5462}.
\newblock


\bibitem[Yan et~al\mbox{.}(2021)]%
        {yan2021two}
\bibfield{author}{\bibinfo{person}{Yujun Yan}, \bibinfo{person}{Milad Hashemi},
  \bibinfo{person}{Kevin Swersky}, \bibinfo{person}{Yaoqing Yang}, {and}
  \bibinfo{person}{Danai Koutra}.} \bibinfo{year}{2021}\natexlab{}.
\newblock \showarticletitle{Two Sides of the Same Coin: Heterophily and
  Oversmoothing in Graph Convolutional Neural Networks}.
\newblock \bibinfo{journal}{\emph{arXiv preprint arXiv:2102.06462}}
  (\bibinfo{year}{2021}).
\newblock


\bibitem[Zhang and Chen(2018)]%
        {zhang2018link}
\bibfield{author}{\bibinfo{person}{Muhan Zhang} {and} \bibinfo{person}{Yixin
  Chen}.} \bibinfo{year}{2018}\natexlab{}.
\newblock \showarticletitle{Link prediction based on graph neural networks}.
\newblock \bibinfo{journal}{\emph{Advances in Neural Information Processing
  Systems}}  \bibinfo{volume}{31} (\bibinfo{year}{2018}),
  \bibinfo{pages}{5165--5175}.
\newblock


\bibitem[Zhang et~al\mbox{.}(2021)]%
        {zhang2021magnet}
\bibfield{author}{\bibinfo{person}{Xitong Zhang}, \bibinfo{person}{Yixuan He},
  \bibinfo{person}{Nathan Brugnone}, \bibinfo{person}{Michael Perlmutter},
  {and} \bibinfo{person}{Matthew Hirn}.} \bibinfo{year}{2021}\natexlab{}.
\newblock \showarticletitle{Magnet: A neural network for directed graphs}.
\newblock \bibinfo{journal}{\emph{Advances in Neural Information Processing
  Systems}}  \bibinfo{volume}{34} (\bibinfo{year}{2021}).
\newblock


\bibitem[Zhou et~al\mbox{.}(2005)]%
        {zhou2005learning}
\bibfield{author}{\bibinfo{person}{Dengyong Zhou}, \bibinfo{person}{Jiayuan
  Huang}, {and} \bibinfo{person}{Bernhard Sch{\"o}lkopf}.}
  \bibinfo{year}{2005}\natexlab{}.
\newblock \showarticletitle{Learning from labeled and unlabeled data on a
  directed graph}. In \bibinfo{booktitle}{\emph{Proceedings of the 22nd
  international conference on Machine learning}}. \bibinfo{pages}{1036--1043}.
\newblock


\bibitem[Zhu et~al\mbox{.}(2021)]%
        {zhu2021relationship}
\bibfield{author}{\bibinfo{person}{Jiong Zhu}, \bibinfo{person}{Junchen Jin},
  \bibinfo{person}{Donald Loveland}, \bibinfo{person}{Michael~T Schaub}, {and}
  \bibinfo{person}{Danai Koutra}.} \bibinfo{year}{2021}\natexlab{}.
\newblock \showarticletitle{On the Relationship between Heterophily and
  Robustness of Graph Neural Networks}.
\newblock \bibinfo{journal}{\emph{arXiv preprint arXiv:2106.07767}}
  (\bibinfo{year}{2021}).
\newblock


\bibitem[Zhu et~al\mbox{.}(2020a)]%
        {zhu2020graph}
\bibfield{author}{\bibinfo{person}{Jiong Zhu}, \bibinfo{person}{Ryan~A Rossi},
  \bibinfo{person}{Anup Rao}, \bibinfo{person}{Tung Mai},
  \bibinfo{person}{Nedim Lipka}, \bibinfo{person}{Nesreen~K Ahmed}, {and}
  \bibinfo{person}{Danai Koutra}.} \bibinfo{year}{2020}\natexlab{a}.
\newblock \showarticletitle{Graph Neural Networks with Heterophily}.
\newblock \bibinfo{journal}{\emph{arXiv preprint arXiv:2009.13566}}
  (\bibinfo{year}{2020}).
\newblock


\bibitem[Zhu et~al\mbox{.}(2020b)]%
        {zhu2020beyond}
\bibfield{author}{\bibinfo{person}{Jiong Zhu}, \bibinfo{person}{Yujun Yan},
  \bibinfo{person}{Lingxiao Zhao}, \bibinfo{person}{Mark Heimann},
  \bibinfo{person}{Leman Akoglu}, {and} \bibinfo{person}{Danai Koutra}.}
  \bibinfo{year}{2020}\natexlab{b}.
\newblock \showarticletitle{Beyond Homophily in Graph Neural Networks: Current
  Limitations and Effective Designs}.
\newblock \bibinfo{journal}{\emph{Advances in Neural Information Processing
  Systems}}  \bibinfo{volume}{33} (\bibinfo{year}{2020}).
\newblock


\bibitem[Zhuo et~al\mbox{.}(2024)]%
        {zhuo2024partitioning}
\bibfield{author}{\bibinfo{person}{Wei Zhuo}, \bibinfo{person}{Zemin Liu},
  \bibinfo{person}{Bryan Hooi}, \bibinfo{person}{Bingsheng He},
  \bibinfo{person}{Guang Tan}, \bibinfo{person}{Rizal Fathony}, {and}
  \bibinfo{person}{Jia Chen}.} \bibinfo{year}{2024}\natexlab{}.
\newblock \showarticletitle{Partitioning message passing for graph fraud
  detection}. In \bibinfo{booktitle}{\emph{The Twelfth International Conference
  on Learning Representations}}.
\newblock


\bibitem[Zhuo and Tan(2022a)]%
        {zhuo2022efficient}
\bibfield{author}{\bibinfo{person}{Wei Zhuo} {and} \bibinfo{person}{Guang
  Tan}.} \bibinfo{year}{2022}\natexlab{a}.
\newblock \showarticletitle{Efficient graph similarity computation with
  alignment regularization}.
\newblock \bibinfo{journal}{\emph{Advances in Neural Information Processing
  Systems}}  \bibinfo{volume}{35} (\bibinfo{year}{2022}),
  \bibinfo{pages}{30181--30193}.
\newblock


\bibitem[Zhuo and Tan(2022b)]%
        {zhuo2022proximity}
\bibfield{author}{\bibinfo{person}{Wei Zhuo} {and} \bibinfo{person}{Guang
  Tan}.} \bibinfo{year}{2022}\natexlab{b}.
\newblock \showarticletitle{Proximity Enhanced Graph Neural Networks with
  Channel Contrast.}. In \bibinfo{booktitle}{\emph{IJCAI}}.
  \bibinfo{pages}{2448--2455}.
\newblock


\bibitem[Zhuo and Tan(2023)]%
        {zhuo2023graph}
\bibfield{author}{\bibinfo{person}{Wei Zhuo} {and} \bibinfo{person}{Guang
  Tan}.} \bibinfo{year}{2023}\natexlab{}.
\newblock \showarticletitle{Graph contrastive learning with adaptive
  proximity-based graph augmentation}.
\newblock \bibinfo{journal}{\emph{IEEE Transactions on Neural Networks and
  Learning Systems}} (\bibinfo{year}{2023}).
\newblock


\bibitem[Z{\"u}gner and G{\"u}nnemann(2019)]%
        {zugner_adversarial_2019}
\bibfield{author}{\bibinfo{person}{Daniel Z{\"u}gner} {and}
  \bibinfo{person}{Stephan G{\"u}nnemann}.} \bibinfo{year}{2019}\natexlab{}.
\newblock \showarticletitle{Adversarial Attacks on Graph Neural Networks via
  Meta Learning}. In \bibinfo{booktitle}{\emph{International Conference on
  Learning Representations (ICLR)}}.
\newblock


\end{thebibliography}

\newpage
\appendix

\section{Heterophily $\neq$ Heterogeneity} \label{app:notation}
We clarify several easily-confounding notations whose prefix is `Hetero-'. This paper studies the graph with {\bf Heterophily} (i.e., heterophilous graph), which is a pattern of the structure-label correlation, i.e., adjacent nodes belong to different classes. Whereas homophily means that adjacent nodes belong to the same class.

{\bf Heterogeneity} is a network concept that is completely unrelated to heterophily. Formally, a network is heterogeneous if it has at least two types of nodes and different relationships between them such as knowledge graphs, and homogeneous if it has a single type of nodes and a single type of edges~\cite{zhu2020beyond}.
\section{Spectral Convolution of DiL-GCN} \label{app:spectral}
Since $\tilde{\mathbf{P}}_{fpr} \in \mathbb{R}^{N \times N}$ is a transition matrix of the combined graph $\tilde{G}_c$, following the inherent property of Markov chain~\cite{poole2014linear}, the eigenvalues $\tilde{\mathbf{P}}_{fpr}$ are bounded in $(-1,1]$. We first analyze the eigenvalues of $\mathcal{T}$. For clarity, we rewrite $\mathcal{T}$ in \cref{eq:norm_dig} as:
\begin{equation}
\mathcal{T} = \tilde{\mathbf{D}}_{c}^{-1} - \frac{1}{2} \left(\Pi^{\frac{1}{2}} \tilde{\mathbf{P}}_{fpr} \Pi^{-\frac{1}{2}} + \Pi^{-\frac{1}{2}} \tilde{\mathbf{P}}_{fpr}^\top \Pi^{\frac{1}{2}}\right).
\end{equation}
Let $\tilde{\mathbf{P}}_{fpr}^* = \Pi^{\frac{1}{2}} \tilde{\mathbf{P}}_{fpr} \Pi^{-\frac{1}{2}}$, then $\tilde{\mathbf{P}}_{fpr} = \Pi^{-\frac{1}{2}} \tilde{\mathbf{P}}_{fpr}^* \Pi^{\frac{1}{2}}$. Let $\lambda$ denote a eigenvalue of $\tilde{\mathbf{P}}_{fpr}$ with the eigenvector $x$, i.e., $\tilde{\mathbf{P}}_{fpr} x = \lambda x$. Then, $\Pi^{-\frac{1}{2}} \tilde{\mathbf{P}}_{fpr}^* \Pi^{\frac{1}{2}} x = \lambda x$, and $\tilde{\mathbf{P}}_{fpr}^* \Pi^{\frac{1}{2}} x = \lambda \Pi^{\frac{1}{2}} x$. Therefore, eigenvalues of $\tilde{\mathbf{P}}^*_{fpr}$ are the same with $\tilde{\mathbf{P}}_{fpr}$ ranging in $(-1,1]$, with eigenvectors $\Pi^{\frac{1}{2}} x$. Thus, the eigenvalues $\lambda_{\mathcal{T}}$ of $\mathcal{T}$ is lower bounded by $\frac{1}{\tilde{d}_{max}} - 1$ and upper bounded by $\frac{1}{\tilde{d}_{min}} + 1$.

Since $\mathcal{T}$ is a real symmetric matrix, it has a complete set of orthonormal eigenvectors $\boldsymbol{U}=\left(\boldsymbol{u}_{1}, \boldsymbol{u}_{2}, \ldots, \boldsymbol{u}_{n}\right)$. Let $\boldsymbol{\Lambda}$ be a diagonal matrix of eigenvalues with $\boldsymbol{\Lambda}(k,k) = \lambda_{\mathcal{T}}(k)$. Since $\boldsymbol{U}$ is unitary, taking the eigenvectors of $\mathcal{T}$ as a set of bases, graph Fourier transform of a signal $f \in \mathbb{R}^N$ on $\tilde{G}_c$ is defined as $\hat{f}=\boldsymbol{U}^{\top} f$, so the inverse graph Fourier transform is:
\begin{equation}
    f=\boldsymbol{U}^{\top} \hat{f} = \sum^N_{k = 1} \hat{f}(k) \boldsymbol{u}_k.
\end{equation}
According to convolution theorem, convolution in Euclidean space corresponds to pointwise multiplication in the Fourier basis. Denoting with $g$ the convolution kernel, the convolution of $f$ with the filter $g$ in the Fourier domain can be defined by:
\begin{equation}
    f * g=\boldsymbol{U}\left(\left(\boldsymbol{U}^{\top} g\right) \odot\left(\boldsymbol{U}^{\top} f\right)\right),
    \label{eq:spectral_conv}
\end{equation}
where $\odot$ is the element-wise Hadamard product, and $*$ is the graph convolution operator. Since the filter $g$ is free, the vector $\boldsymbol{U}^{\top} g$ can be any vector. Thus, the filter can be replaced by a diagonal matrix $\boldsymbol{\Sigma} = diag(\boldsymbol{\theta})$ parameterized by $\boldsymbol{\theta} \in \mathbb{R}^N$. We can rewrite \cref{eq:spectral_conv} as $\boldsymbol{U} \boldsymbol{\Sigma} \boldsymbol{U}^{\top} f$.

To reduce the number of tranable parameters to prevent overfitting and avoid explicit diagonalization of the matrix $\mathcal{T}$. Following ChebyNet~\cite{defferrard2016convolutional}, which restricts spectral convolution kernel $\boldsymbol{\Sigma}$ to a polynomial expansion of $\boldsymbol{\Lambda}$ and approximates the graph spectral convolutions by a truncated expansion in terms of Chebyshev polynomials up to $K$-th order, we define a normalized eigenvalue matrix, with entries in $(-1,1)$, by $\tilde{\boldsymbol{\Lambda}} = \frac{\boldsymbol{\Lambda}-\tilde{\mathbf{D}}_{c}^{-1}}{2}$, which can be verified by $ \frac{1}{\tilde{d}_{max}} - 1 \leq \lambda_{\mathcal{T}} \leq \frac{1}{\tilde{d}_{min}} + 1$. Then the graph convolution can be approximated by
\begin{equation}
    \boldsymbol{U} \Sigma \boldsymbol{U}^{T} f \approx \sum_{k=0}^{K} \theta_{k} T_{k}\left(\frac{\mathcal{T} - \tilde{\mathbf{D}}_{c}^{-1}}{2}\right)f,
\end{equation}
where $T_k(\cdot)$ is the kth order matrix Chebyshev polynomial defined by $T_0(x) = 1$, $T_1(x) = x$, and $T_{k}(x)=2 x T_{k-1}(x)+T_{k-2}(x)$ for $k \geq 2$, $\theta_0, \cdots,\theta_k$ real-valued parameters. Let $\boldsymbol{Y} = \boldsymbol{U} \Sigma \boldsymbol{U}^{T}$, we employ an affine appoximation ($K=1$) with coefficients $\theta_0 = \theta$ and $\theta_1 = 2\theta$, from which we attain the graph convolution operation:
\begin{equation}
    \boldsymbol{Y} f \approx \theta \left(\mathbf{I} + \frac{1}{2} \left(\Pi^{\frac{1}{2}} \tilde{\mathbf{P}}_{fpr} \Pi^{-\frac{1}{2}} + \Pi^{-\frac{1}{2}} \tilde{\mathbf{P}}_{fpr}^\top \Pi^{\frac{1}{2}}\right)\right)f
    \label{eq:basic_conv}
\end{equation}
In our final form, we replace identity matrix $I$ in \cref{eq:basic_conv} with $\tilde{\mathbf{D}}_{c}^{-1}$ as shown in \cref{eq:aug_DiLaplacian}. The reason is that in \cref{eq:basic_conv} the transition matrix is $\tilde{\mathbf{P}}_{fpr} + \mathbf{I}$, which is equivalent to giving each node a large probability to move to itself in a random walk on $\tilde{G}_c$. It makes the model hard to capture enough information from neighbors. Hence, we assume that the probability of the node moving to itself is the same as moving to its neighbors in each step, and it can be satisfied by $\widehat{\mathcal{T}}$ in \cref{eq:aug_DiLaplacian}.

\section{Proof of Theorem 3.1}\label{app:proof}

\begin{proof}
As $\tilde{\mathbf{P}}_{fpr}^t \mathbf{J} = \mathbf{J}$, we have $\mathbf{Z}\mathbf{J} = \boldsymbol{0}_{n \times n}$. Let $\mathfrak{T} = \Pi^{-\frac{1}{2}} \tilde{\mathbf{T}} \Pi^{-\frac{1}{2}} = \Pi^{\frac{1}{2}} (\tilde{\mathbf{D}}_{c}^{-1}-\tilde{\mathbf{P}}_{fpr}) \Pi^{-\frac{1}{2}}$, $\mathcal{J} =\Pi^{\frac{1}{2}} \mathbf{J} \Pi^{\frac{1}{2}}$ and $\mathcal{Z} = \Pi^{\frac{1}{2}}\mathbf{Z}\Pi^{-\frac{1}{2}}$, from \cref{eq:fm} we have:
\begin{equation}
    \mathcal{Z} + \mathcal{J} = (\mathfrak{T} + \mathcal{J}-\tilde{\mathbf{D}}_{c}^{-1} + \mathbf{I}) ^{-1},
\end{equation}
then multiplying from the right by $(\mathfrak{T} + \mathcal{J}-\tilde{\mathbf{D}}_{c}^{-1} + \mathbf{I})$, we have:
\begin{equation}
\begin{aligned}
    \mathbf{I} &=(\mathcal{Z} + \mathcal{J})(\mathfrak{T} + \mathcal{J}-\tilde{\mathbf{D}}_{c}^{-1} + \mathbf{I}) \\
    & = \mathcal{Z}\mathfrak{T} + \mathcal{Z}\mathcal{J}-\mathcal{Z}\tilde{\mathbf{D}}_{c}^{-1} + \mathcal{Z} +\mathcal{J}\mathfrak{T} + \mathcal{J}^2-\mathcal{J}\tilde{\mathbf{D}}_{c}^{-1} + \mathcal{J}.
\end{aligned}
\label{eq:fm2}
\end{equation}
Since $\mathcal{J}^2 = \mathcal{J}$, $\mathcal{J}\mathfrak{T} = \mathcal{J}(\tilde{\mathbf{D}}_{c}^{-1} - \mathbf{I})$, and $\mathcal{Z}\mathcal{J} =  \Pi^{\frac{1}{2}} \mathbf{Z}\mathbf{J} \Pi^{\frac{1}{2}} = \boldsymbol{0}_{n \times n}$, \cref{eq:fm2} can be simplified to:
\begin{equation}
    \mathcal{Z}(\mathfrak{T}-\tilde{\mathbf{D}}_{c}^{-1} + \mathbf{I}) = \mathbf{I}-\mathcal{J}
    \label{eq:fm3}
\end{equation}
Similarly, multiplying from the left we have $(\mathfrak{T}-\tilde{\mathbf{D}}_{c}^{-1} + \mathbf{I})\mathcal{Z} = \mathbf{I}-\mathcal{J}$. As $\pi^\top \mathbf{Z} = \boldsymbol{0}$, we have $\mathcal{J}\mathcal{Z} = 0$. Thus, $\mathcal{Z}(\mathfrak{T}-\tilde{\mathbf{D}}_{c}^{-1} + \mathbf{I}) \mathcal{Z}=\mathcal{Z}$. Furthermore, it is easy to derive that $(\mathfrak{T}-\tilde{\mathbf{D}}_{c}^{-1} +\mathbf{I})\mathcal{J} = 0$, then have $(\mathfrak{T}-\tilde{\mathbf{D}}_{c}^{-1} + \mathbf{I}) \mathcal{Z} (\mathfrak{T}-\tilde{\mathbf{D}}_{c}^{-1} + \mathbf{I}) = (\mathfrak{T}-\tilde{\mathbf{D}}_{c}^{-1} + \mathbf{I})$. Besides, the left part of \cref{eq:fm3} symmetric, so $(\mathcal{Z}(\mathfrak{T}-\tilde{\mathbf{D}}_{c}^{-1} + \mathbf{I}))^\top = \mathcal{Z}(\mathfrak{T}-\tilde{\mathbf{D}}_{c}^{-1} + \mathbf{I})$. Similarly, $((\mathfrak{T}-\tilde{\mathbf{D}}_{c}^{-1} + \mathbf{I})\mathcal{Z})^\top = (\mathfrak{T}-\tilde{\mathbf{D}}_{c}^{-1} + \mathbf{I})\mathcal{Z}$. These facts satisfy the sufficient conditions of Moore–Penrose pseudoinverse, such that 
\begin{equation}
    \mathcal{Z} = (\mathfrak{T}-\tilde{\mathbf{D}}_{c}^{-1} + \mathbf{I}) ^{\dagger}.
    \label{eq:fm4}
\end{equation}
Finally, recovering $\mathcal{Z}$ and $\mathfrak{T}$, which concludes the proof.
\end{proof}

\section{Implementation Details}\label{app:id}
\paragraph{Hardware infrastructures.} 
The experiments are conducted on Linux servers installed with a NVIDIA Quadro RTX8000 GPU and ten Intel(R) Xeon(R) Silver 4210R CPUs.

\paragraph{Hyperparameter Specifications}
All parameters of our DiL-GCN and DiL-GCN$_{CT}$ are initialized with Glorot initialization~\cite{glorot2010understanding}, and trained using Adam optimizer~\cite{kingma2014adam} with learning rate $lr$ selected from $\{0.01, 0.005\}$. The activation function $\sigma$ is $\mathrm{ReLU}$. The $l_2$ weight decay is selected from $\{0, 5e-4, 8e-4 ,1e-3\}$, and the dropout rate~\cite{srivastava2014dropout} is selected from $\{0.5, 0.6, 0.7\}$. The hyperparameter $window\_size$ of Feature-aware PageRank is set to a non-zero even number, which selected from $\{1,2,3\}$. The number of iterations $t$ of the power method to compute the stationary distribution $\pi$ is set to 30 for all datasets. The dimension of hidden layers is selected from $\{48, 64, 96\}$ for all datasets and the number of layers is set to 2 for all datasets. For DiL-GCN$_{CT}$, we set $\mu$ to 0.97 for all datasets to ensure the sparsity of the graph propagation matrix. In addition, we use an early stopping strategy on accuracies on the validation nodes, with a patience 500 epochs. All dataset-specific hyperparameter configurations are summarized in \Cref{tab:hyperparameter}.

\begin{table}[h]
\centering
\caption{Hyperparameter specifications.}
	\begin{tabular}{cccccccccc}
			\toprule
			Dataset & $window\_size$  &  $lr$ &\makecell{Weight\\decay}  & \makecell{Hidden\\dimension}   \\
			\midrule
			Texas        &  1      & 0.01     & 8e-4    & 48                      \\
			Wisconsin    &  2      & 0.01     & 1e-3    & 96                        \\
			Actor		 &  1      & 0.01     & 1e-3    & 64                       \\
			Squirrel     &  1      & 0.005    & 5e-4    & 64                        \\
			Chameleon	 &  1      & 0.01     & 5e-4    & 64                       \\
			Cornell      &  3      & 0.01     & 1e-3    & 64                         \\
			deezer       &  1      & 0.01     & 5e-4    & 64                        \\  
			Citeseer     &  2      & 0.01     & 8e-4    & 64                           \\
			CoraML       &  3      & 0.01     & 5e-4    & 96                          \\
			CoauthorCS   &  1      & 0.01     & 5e-4    & 64                          \\
			\bottomrule
	\end{tabular}
\label{tab:hyperparameter}
\end{table}

\end{document}